\documentclass{article}
\usepackage[utf8]{inputenc} 
\usepackage[T1]{fontenc}    

\usepackage{microtype}
\usepackage{graphicx}
\usepackage{subfigure}
\usepackage{booktabs} 
\usepackage{multirow}
\usepackage{hyperref}
\usepackage{natbib}
\usepackage{fullpage}
\usepackage{authblk}

\usepackage{amsmath}
\usepackage{amssymb}
\usepackage{mathtools}
\usepackage{amsthm}
\usepackage{algorithm}
\usepackage{algpseudocode}

\usepackage[capitalize,noabbrev]{cleveref}

\theoremstyle{plain}
\newtheorem{theorem}{Theorem}[section]
\newtheorem{proposition}[theorem]{Proposition}
\newtheorem{lemma}[theorem]{Lemma}
\newtheorem{corollary}[theorem]{Corollary}
\theoremstyle{definition}
\newtheorem{definition}[theorem]{Definition}
\newtheorem{assumption}[theorem]{Assumption}
\theoremstyle{remark}
\newtheorem{remark}[theorem]{Remark}

\newcommand{\A}{\mathcal{A}}
\newcommand{\M}{\mathcal{M}}
\newcommand{\N}{\mathcal{N}}
\newcommand{\R}{\mathbb{R}}
\newcommand{\lap}[1]{\operatorname{Lap}\left(#1\right)}

\newcommand{\epsg}{\epsilon_g}
\newcommand{\epsH}{\epsilon_H}

\usepackage{apptools}
\AtAppendix{\counterwithin{theorem}{section}}

\title{Differentially Private Optimization for Smooth Nonconvex ERM}
\author{Changyu Gao\thanks{Department of Industrial and Systems Engineering, University of Wisconsin-Madison, WI 53706, USA \newline (email: changyu.gao@wisc.edu)} \ and Stephen J. Wright\thanks{Department of Computer Sciences, University of Wisconsin-Madison, WI 53706, USA (email: swright@cs.wisc.edu)}}
\date{}

\begin{document}
\maketitle

\begin{abstract}
    We develop simple differentially private optimization algorithms that move along directions of (expected) descent to find an approximate second-order solution for nonconvex ERM.\@ We use line search, mini-batching, and a two-phase strategy to improve the speed and practicality of the algorithm. Numerical experiments demonstrate the effectiveness of these approaches.
\end{abstract}

\section{Introduction}
Privacy protection has become a central issue in machine learning algorithms,
and  differential privacy \citep{dworkAlgorithmicFoundationsDifferential2014} is a rigorous and popular framework for quantifying privacy.
We propose a differentially private optimization algorithm that finds an approximate second-order solution for (possibly nonconvex) ERM problems. 
We propose several techniques to improve the practical performance of the method, including backtracking line search, mini-batching, and a heuristic to avoid the effects of conservative assumptions made in the analysis.

For given \(f : \R^d \to \R\), smooth but nonconvex, we consider the following minimization problem:
\[
    \min_{w \in \R^d} f(w).
\]
We seek an approximate second-order solution, defined  as follows.
\begin{definition}[\((\epsg,\epsH)\)-2S]
    For given positive values of \(\epsg\) and \(\epsH\), We say that \(w\) is an \((\epsg, \epsH)\)-approximate second-order solution (abbreviated as \((\epsg,\epsH)\)-2S) if
    \begin{equation} \label{eq:2on}
        \|\nabla f(w)\| \le \epsg, \quad \lambda_{\text{min}}\left(\nabla^2 f(w) \right) \ge -\epsH.
    \end{equation}
\end{definition}

We are mostly interested in the case of \(\epsg = \alpha\) and \(\epsH = \sqrt{M \alpha}\), for some \(\alpha>0\).
That is, we seek an \((\alpha, \sqrt{M \alpha})\)-2S, where \(M\) is the Lipschitz constant for \( \nabla^2 f\).


We will focus on empirical risk minimization (ERM).
\begin{definition}[ERM]\label{def:erm}
    Given a dataset \(D=\left\{x_{1}, \ldots, x_{n}\right\}\) and a loss function \(l(w,x)\), we seek the parameter \(w \in \mathbb{R}^d\) that minimizes the empirical risk
    \begin{equation} \label{eq:ERM}
        f(w) = L(w, D) := \frac{1}{n} \sum_{i=1}^{n} l \left(w, x_{i}\right).
    \end{equation}
\end{definition}

ERM is a classical problem in machine learning that has been studied extensively; see, for example \citet{shalev-shwartzUnderstandingMachineLearning2014}.
In this paper, we describe differentially private (DP) techniques for solving ERM.\@

Previous research on DP algorithms for ERM and optimization has focused mainly on convex loss functions.
Recent research on differentially private algorithms for nonconvex ERM \citep{wangDifferentiallyPrivateEmpirical2018,wangDifferentiallyPrivateEmpirical2019a,zhangEfficientPrivateERM2017} targets an approximate stationary point, which satisfies only the first condition in~\eqref{eq:2on}.
\citet{wangEscapingSaddlePoints2021} proposes a trust-region type (DP-TR) algorithm that gives an approximate second-order  solution for ERM, satisfying both conditions in~\eqref{eq:2on}, for particular choices of \(\epsg\) and \(\epsH\).
This work requires the trust-region subproblem to be solved exactly at each iteration, and fixes the radius of the trust region at a small value, akin to a ``short step'' in a line-search method.
An earlier work \citep{wangDifferentiallyPrivateEmpirical2019} proposed the DP-GD algorithm, which takes short steps in a noisy gradient direction, then sorts through all the iterates so generated to find one that satisfies second-order necessary conditions.
Our approach matches the sample complexity bound in DP-GD, which is \(O\left(\frac{\sqrt{d}}{\alpha^2\sqrt{\rho}}\right)\) for \(\rho\)-\(z\)CDP or \(O\left(\frac{\sqrt{d\ln(1/\delta)}}{\alpha^2\varepsilon}\right)\) for \((\varepsilon, \delta)\)-DP for finding an \((\alpha, \sqrt{M\alpha})\)-2S, and has an iteration complexity of  \(O(\alpha^{-2})\).

Our contributions can be summarized as follows.
\begin{itemize}
    \item Our algorithm is elementary and is based on a simple (non-private) line-search algorithm for finding an approximate second-order  solution.
          It evaluates second-order information (a noisy Hessian matrix) only when insufficient progress can be made using first-order (gradient) information alone.
          By contrast, DP-GD uses the (noisy) Hessian only for checking the second-order approximate condition, while DP-TR requires the noisy Hessian to be calculated at every iteration.
    \item Our algorithm is practical and fast. DP-TR has a slightly better worst-case sample complexity bound than our method, depending on \(O(\alpha^{-7/4})\) rather than \(O(\alpha^{-2})\).
          However, our method makes use of line searches, allowing it to adapt to the local geometry of the function and thus attain better practical performance than worst-case bounds would suggest.
         By contrast, DP-TR uses a small trust region  whose size is based on the worst-case global properties of the function, and requires exact solution of its trust-region subproblem. This operation is relatively expensive and also undesirable when the gradient is large enough to take a productive step without any need to utilize the Hessian.
          Experiments demonstrate that our algorithm requires fewer iterations than DP-TR and does less computation on average at each iteration, and thus runs significantly faster than DP-TR in practice.\@
          In any case, the \emph{mini-batch} version of DP-TR has a sample complexity \(O(\alpha^{-2})\), matching the sample complexity of the mini-batch version of our algorithm.
    \item We use line search and mini-batching to accelerate the algorithm. Differentially private line search algorithms have been proposed by \citep{chenStochasticAdaptiveLine2020}.
          We use the same sparse vector technique as used by their work, but provide a tighter analysis of the sensitivity of the query for checking the sufficient decrease condition. In addition, we provide a rigorous analysis of the guaranteed function decrease with high probability.
    \item  To complement our worst-case analysis,  we propose a heuristic that can obtain much more rapid convergence while retaining the guarantees provided by the analysis.
\end{itemize}

The remainder of the paper is structured as follows. In Section 2, we review basic definitions and properties from differential privacy, and outline our assumptions about the function \(f\) to be optimized.  In Section 3, we describe our algorithm and its analysis. We will discuss the basic short step version of the algorithm in Section 3.1 and an extension to a practical line search method in Section 3.2. A mini-batch adaptation of the algorithm is described in Section 3.3. In Section 4, we present experimental results and demonstrate the effectiveness of our algorithms.

\section{Preliminaries}
We use several variants of DP in our analysis, including \((\varepsilon, \delta)\)-DP \citep{dworkAlgorithmicFoundationsDifferential2014}, \((\alpha, \epsilon)\)-RDP \citep{mironovRenyiDifferentialPrivacy2017}, and \(z\)CDP \citep{bunConcentratedDifferentialPrivacy2016}.
We review their definitions and properties in Appendix~\ref{sec:prelim-dp}.

We make the following assumptions about the smoothness of the objective function \(f\).
\begin{assumption}
    We assume \(f\) is lower bounded by \(\underline f\). Assume further that \(f\) is \(G\)-smooth and has \(M\)-Lipschitz Hessian, that is, for all \(w_1, w_2 \in \operatorname{dom}(f)\), we have
    \begin{align*}
        \|\nabla f(w_1) - \nabla f(w_2)\|     & \le G \|w_1 - w_2\|, \\
        \|\nabla^2 f(w_1) - \nabla^2 f(w_2)\| & \le M \|w_1 - w_2\|,
    \end{align*}
\end{assumption}
where \(\|\cdot\|\) denotes the vector \(2\)-norm and the matrix \(2\)-norm respectively. We use this notation throughout the paper.

For the ERM version of \(f\) (see Definition~\ref{def:erm}), we make additional assumptions.
\begin{assumption}
    For the ERM setting~\eqref{eq:ERM}, we assume the loss function is \(l(w, x)\) is \(G\)-smooth and has \(M\)-Lipschitz Hessian with respect to \(w\). Thus \(L(w, D)\) (the average loss across \(n\) samples) is also \(G\)-smooth and has \(M\)-Lipschitz Hessian with respect to \(w\). In addition, we assume \(l(w, x)\) has bounded function values, gradients, and Hessians. That is, there are constants \(B\), \(B_g\), and \(B_H\) such that for any \(w, x\) we have,
    \begin{equation*}
        0 \le l(w, x) \le B, \; \|\nabla_w l(w, x)\| \le B_g, \; \|\nabla^2_w l(w, x)\| \le B_H.
    \end{equation*}
    As a consequence, the \(l^2\) sensitivity of \(L(w, D)\) and \(\nabla L(w, D)\) is bounded by \(B / n\) and \(2B_g / n\) respectively. We have
    \begin{equation*}
    \begin{split}
        &\|\nabla^2 L(w, D) - \nabla^2 L(w, D')\|_F \\
        &\le \sqrt{d} \, \|\nabla^2 L(w, D) - \nabla^2 L(w, D')\| \le \frac{2B_H\sqrt{d}}{n}.
    \end{split}
    \end{equation*}
    To simplify notation, we define \(g(w) := \nabla f(w)\) and \(H(w) := \nabla^2 f(w)\).
    From the definition~\eqref{eq:sens} of \(\ell_2\)-sensitivity,   the sensitivities of \(f\), \(g\), and \(H\) are
    \begin{equation} \label{eq:sensitivity}
        \Delta_f = \frac{B}{n}, \quad \Delta_g = 2\frac{B_g}{n}, \quad \Delta_H = \frac{2B_H\sqrt{d}}{n}.
    \end{equation}
\end{assumption}

\section{Algorithms and their Properties}

Our algorithmic starting point is the elementary algorithm described in \citet[Chapter~3.6]{wrightOptimizationDataAnalysis2022} that has convergence guarantees to points that satisfy approximate second-order  conditions.
For simplicity, we use the following notation to describe and analyze the method:
\begin{equation}
\begin{split}
    f_{k} &:= f(w_k), \\
    g_{k} &:= g(w_k) = \nabla f(w_k), \\
    H_{k} &:= H(w_k) = \nabla^{2} f(w_k).
\end{split}
\end{equation}
We employ the Gaussian mechanism to perturb gradients and Hessians, and denote
\[
    \tilde g_{k}=g_{k}+\varepsilon_{k}, \quad \tilde H_k = H_k + E_k,
\]
where \(\varepsilon_k \sim \mathcal{N}(0, \Delta_g^2 \sigma_g^2 I_d)\) for some chosen parameter \(\sigma_g\) and \(E_k\) is a symmetric matrix in which
each entry on and above its diagonal is i.i.d.\ as \(\mathcal{N}\left(0, \Delta_H^2 \sigma_H^2\right)\), for some chosen value of \(\sigma_H\).
Let \(\tilde \lambda_k\) denote the minimum eigenvalue of \(\tilde H_k\) with the corresponding eigenvector \(\tilde p_k\) of which the sign and norm are chosen to satisfy
\begin{equation}\label{eq:min_eig}
    \|\tilde p_k\| = 1 \quad \text{and} \quad (\tilde p_k)^T \tilde g_{k} \le 0.
\end{equation}

Algorithm~\ref{alg:dpopt} specifies the general form of our optimization algorithm. We will discuss two strategies --- a ``short step'' strategy and one based on backtracking line search ---  to choose the step sizes \(\gamma_{k, g}\) and \(\gamma_{k, H}\) to be taken along the directions \(\tilde g_k\) and \(\tilde p_k\), respectively.
For each variant, we define a quantity MIN\_DEC to be the minimum decrease, and use it together with a specific lower bound on \(f\) to define an upper bound \(T\) of the required number of iterations.
In each iteration, we take a step in the negative of the perturbed gradient direction \(\tilde g_k\) if \(\|\tilde g_k\| > \epsg\).
Otherwise, we check the minimum eigenvalue \(\tilde \lambda_k\) of the perturbed Hessian \(\tilde H_k\). If \(\tilde \lambda_k < -\epsH\), we take a step along the direction \(\tilde p_k\).
In the remaining case, we have \(\|\tilde g_k\| \le \epsg\) and \(\tilde \lambda_k \ge -\epsH\), so the approximate second-order  conditions are satisfied (up to some constant factors) and we output the current iterate \(w_k\) as a 2S solution.

\begin{algorithm}[!htb]
    \caption{DP Optimization with Second-Order Guarantees (General form)}\label{alg:dpopt}
    \begin{algorithmic}
        \State \textbf{Given:} minimum decrease per iteration MIN\_DEC, tolerances \(\epsilon_g\), \(\epsilon_H\), noise parameters \(\sigma_f\), \(\sigma_g\), \(\sigma_H\)
        \vspace{0.1in}
        \State Initialize \(w_0\) and sample \(z \sim \mathcal{N}(0, \Delta_f^2 \sigma_f^2)\)
        \State Compute an upper bound of the required number of iterations as follows
        \begin{equation} \label{eq:T}
            T = \left\lceil\frac{f(w_0) + |z| - \underline{f}}{\text{MIN\_DEC}} \right\rceil
        \end{equation}
        \State \(\sigma_g\) and \(\sigma_H\) using \(T\) (See theorems for details)
        \For{\(k = 0,1, \ldots, T-1\)}
        \State Sample \(\varepsilon_k \sim \mathcal{N}\left(0, \Delta_g^2 \sigma_g^2 I_d\right)\)
        \State Compute the perturbed gradient \(\tilde g_{k} = g_{k}+\varepsilon_{k}\)
        \If{\(\left\|\tilde g_k\right\| > \epsilon_{g}\)}
            \State Choose step size \(\gamma_{k,g}\) and set \(w_{k+1} \gets w_k -  \gamma_{k, g} \tilde g_{k}\) \Comment{Gradient step}
        \Else
            \State Sample \(E_k\) such that \(E_k\) is a \(d \times d\) symmetric matrix in which
            each entry on and above its diagonal is i.i.d.\ as \(\mathcal{N}\left(0, \Delta_H^2 \sigma_H^2\right)\)
            \State Compute perturbed Hessian \(\tilde H_k = H_k + E_k\)
            \State Compute the minimum eigenvalue of \(\tilde H_k\) and the corresponding eigenvector \((\tilde \lambda_k, \tilde p_k)\) satisfying~\eqref{eq:min_eig}
            \If{\(\tilde \lambda_{k} < -\epsH\)}
                \State Choose step size \(\gamma_{k,H}\) and set \(w_{k+1} \gets w_k + \gamma_{k, H} \tilde p_k\) \Comment{Negative curvature step}
            \Else
                \State {\bfseries return} \(w_k\)
            \EndIf
        \EndIf
        \EndFor
    \end{algorithmic}
\end{algorithm}

The quantities \(\sigma_f\), \(\sigma_g\), \(\sigma_H\) determine the amount of noise added to function, gradient, and Hessian evaluations, respectively, with the goal of preserving privacy via the Gaussian Mechanism.
We can target a certain privacy level for the overall algorithm (\(\rho\) in \(\rho\)-zCDP, for example), find an upper bound on the number of iterations required by whatever variant of Algorithm~\ref{alg:dpopt} we are using, and then choose \(\sigma_f\), \(\sigma_g\), and \(\sigma_H\) to ensure this level of privacy.
Conversely, we can choose positive values for \(\sigma_f\), \(\sigma_g\), and \(\sigma_H\) and then determine what level of privacy can be ensured by this choice.
We can keep track of the privacy leakage as the algorithm progresses, leading to the possibility of adaptive schemes for choosing the noise variances.

\subsection{Short Step}
In the short step version of the algorithm, we choose step sizes as follows 
\begin{equation} \label{eq:short_step}
    \gamma_{k, g} \equiv \frac{1}{G}, \quad \gamma_{k, H} \equiv \frac{2|\tilde\lambda_{k}|}{M}.
\end{equation}
The choices of MIN\_DEC and the noise parameters \(\sigma_f\), \(\sigma_g\), and \(\sigma_H\) are discussed in the following results.

First, we discuss the privacy guarantee and its relationship to the noise variances and the number of iterations.
\begin{theorem}
    Let the noise parameters \(\sigma_f\), \(\sigma_g\), \(\sigma_H\) be given.
    Suppose a run of Algorithm~\ref{alg:dpopt} takes \(k_g\) gradient steps and \(k_H\) negative curvature steps. Then the run is  \(\rho\)-\(z\)CDP where
    \(
    \rho = \frac 12 \left(\frac{1}{\sigma_f^2} +
    \frac{k_g + k_H}{\sigma_g^2} + \frac{k_H}{\sigma_H^2}
    \right).
    \)

    Recall that \(T\) is the maximum number of iterations defined in~\eqref{eq:T}. Let
    \(
    \bar \rho = \frac 12 \left(\frac{1}{\sigma_f^2} +
    \frac{T}{\sigma_g^2} + \frac{T}{\sigma_H^2}
    \right).
    \)
    We always have \(\bar \rho \ge \rho\), so the algorithm is \(\bar \rho\)-\(z\)CDP.\@
    Conversely, for given \(\rho > 0\) and \(c_f \in (0,1) \), we can choose \(\rho_f = c_f \rho\), and
    \begin{equation}
        \sigma_f^2 = \frac{1}{2\rho_f}, \quad \sigma_g^2 = \sigma_H^2 = \frac{T}{(1-c_f) \rho}. \label{eq:short_sigma}
    \end{equation}
    to ensure that the algorithm is \(\rho\)-\(z\)CDP.\@
\end{theorem}
\begin{proof}
    The proof follows directly from the \(z\)CDP guarantee for the Gaussian mechanism combined with postprocessing and composition of \(z\)CDP.\@
\end{proof}

\begin{remark}
    In our algorithm, the actual noise is scaled by the corresponding sensitivity \(\Delta\) defined in~\eqref{eq:sensitivity}. We do the same for later algorithms.
    In practice, we expect most steps to be gradient steps, so \(\bar \rho\) is an overestimate of the actual privacy level \(\rho\). 
    In practice, therefore, we can be more aggressive in choosing the noise variances than this worst-case theory would suggest.
    We discuss a two-phase approach in Section~\ref{sec:improvement}.
\end{remark}

We now discuss guarantees of the output of Algorithm~\ref{alg:dpopt}.
First, we estimate MIN\_DEC in each short step.
\begin{lemma}\label{lem:short_dec}
    With the short step size choices~\eqref{eq:short_step}, if the noise satisfies the following conditions
    for some positive constants \(c\), \(c_1\), and \(c_2\) such that \(c_1 < \tfrac12\) and \(c_{2}+c < \tfrac13\), and
    \begin{subequations}\label{eq:bounded_noise}
        \begin{align}
            \left\|\varepsilon_{k}\right\| & \le \min \left(c_{1} \epsg, \frac{c_{2}}{M} \epsH^{2}\right), \\
            \left\|E_{k}\right\|           & \le c \, \epsH,
        \end{align}
    \end{subequations}
    then the amount of decrease in each step is at least
    \begin{equation} \label{eq:short_dec}
        \textnormal{MIN\_DEC} = \min\left(
        \frac{1-2 c_{1}}{2 G}\epsg^{2}, \, 2\left(\frac{1}{3}-c_{2}-c\right) \frac{\epsH^{3}}{M^{2}}
        \right).
    \end{equation}
    The true gradient and true minimum eigenvalue of the Hessian satisfy the following,
    \begin{equation} \label{eq:close}
        \left\|g_{k}\right\|
        \le (1+c_1) \left\|\tilde{g}_{k}\right\|, \quad
        \lambda_k > - (1+c) |\tilde \lambda_k|.
    \end{equation}
\end{lemma}
\begin{remark}
    The constants \(c\), \(c_1\), and \(c_2\) in~\eqref{eq:bounded_noise} control the accuracy of our noisy gradient and Hessian estimates. MIN\_DEC is smaller when we choose smaller tolerances \(\epsg\) and \(\epsH\). Smaller tolerances also yield a tighter solution, but we will need smaller noise variances to satisfy the conditions~\eqref{eq:bounded_noise}. This requirement translates to a larger required sample size \(n\) for our ERM problem, as we will see in Theorem~\ref{thm:short_thm}.
\end{remark}

\begin{corollary}\label{cor:output}
    Assuming the noise satisfies~\eqref{eq:bounded_noise} at each iteration, the short step version~(using~\eqref{eq:short_step},~\eqref{eq:short_dec}) of the algorithm will output a \(((1 + c_{1}) \epsg, (1+c) \epsH)\)-2S.\@
\end{corollary}

With the results above, we now analyze the guarantees of the fixed step-size algorithm under the ERM setting.

\begin{theorem}[Sample complexity of the short step algorithm]\label{thm:short_thm}
    Consider the ERM setting. Suppose that the number of samples \(n\) satisfies \(n \ge n_{\min}\), where
    \[ \label{min_n}
        n_{\min} :=  \max \left(
        \frac{\sqrt{2d} B_g \sigma_g \log \frac{T}{\zeta}}{\min \left(c_{1} \epsg, \, \frac{c_{2}}{M} \epsH^{2}\right)},
        \frac{C\sqrt{d} B_H \sigma_H \log \frac{T}{\zeta}}{c \, \epsH}
        \right).
    \]
    With probability at least \(\{(1 - \frac{\zeta}{T})(1 - C\exp{(-C_1Cd)})\}^{T}\) where \(C\) and \(C_1\) are universal constants in Lemma~\ref{lem:matrix_concentration},
    the output of the short step version (using~\eqref{eq:short_step},\eqref{eq:short_dec}) of the algorithm is a \(((1+c_{1}) \epsg, (1+c) \epsH)\)-2S.

    With the choice of \(\sigma\)'s in~\eqref{eq:short_sigma}, hiding logarithmic terms and constants, the asymptotic dependence of \(n_{\min}\) on \((\epsg, \epsH)\), \(\rho\) and \(d\), is

    \begin{equation}
        n_{\min} = \frac{\sqrt{d}}{\sqrt{\rho}} \, \tilde O \left(
        \max \left(
            \epsg^{-2}, \epsg^{-1} \epsH^{-2}, \epsH^{-7/2}
            \right)
        \right).
    \end{equation}
\end{theorem}
When  \((\epsg,\epsH)  = (\alpha,\sqrt{M \alpha})\), the dependence simplifies to \(\frac{\sqrt{d}}{\sqrt{\rho}} \, \tilde O (\alpha^{-2})\).

Before proving Theorem~\ref{thm:short_thm}, we introduce two concentration results.
\begin{lemma}[Gaussian concentration, \citep{vershyninHighDimensionalProbabilityIntroduction2018}]
    For \(x \sim \mathcal{N}\left(0, \sigma^{2} I_{d}\right)\), with probability at least \(1-\eta\) for any \(1>\eta>0\), we have
    \[
        \|x\| \leq \sqrt{2 d} \sigma \log \frac{1}{\eta}.
    \]
\end{lemma}

\begin{lemma}[Upper tail estimate for Wigner ensembles {\citep[p. 110]{taoTopicsRandomMatrix2012}}]\label{lem:matrix_concentration}
    Let \(M=(m_{i j})_{1 \le i,j \le d}\) be an \(d \times d\) random symmetric matrix. Suppose that the coefficients \(m_{i j}\) of \(M\) are independent for \(j \geq i\), mean zero, and have uniform sub-Gaussian tails.
    There exist universal constants \(C, C_1>0\) such that for all \(A \geq C\), we have
    \[
        \mathbf{P}\left(\|M\| > A \sqrt{d}\right) \le C \exp (-C_1 A d).
    \]
\end{lemma}

\begin{proof}
(Theorem~\ref{thm:short_thm})  It follows from concentration results that, in iteration \(k\), with probability at least \((1 - \frac{\zeta}{T})(1 - C\exp{(-C_1Cd)})\), we have
\begin{subequations}\label{gauss_concentration}
    \begin{align}
        \|\varepsilon_k\| & \le \sqrt{2d} \Delta_g \sigma_g \log \frac{T}{\zeta}, \\
        \|E_k\|           & \le C \sqrt{d} \Delta_H \sigma_H.
    \end{align}
\end{subequations}
We need to find a condition on \(n\) that ensures that the right-hand sides are less than the right-hand sides of~\eqref{eq:bounded_noise}.
We substitute for \(\Delta_g\) and \(\Delta_H\) from~\eqref{eq:sensitivity} and solve for \(n_{\min}\) by rearranging the terms. The result then follows from Corollary~\ref{cor:output} if the concentration results hold for all iterations.

Now let us calculate the success probability. For each iteration, we have a probability of at least \((1 - \frac{\zeta}{T})(1 - C\exp{(-C_1Cd)})\) that the concentration results hold (if we do not compute the perturbed Hessian, the probability is higher with at least \(1 - \frac{\zeta}{T}\)). Using conditional probability, the overall success probability is \(\{(1 - \frac{\zeta}{T})(1 - C\exp{(-C_1Cd)})\}^{\tau}\) conditioned on the number of iterations \(\tau\). Since \(\tau \le T\), the overall success probability is at least \(\{(1 - \frac{\zeta}{T})(1 - C\exp{(-C_1Cd)})\}^T\).

For the second part, recall from~\eqref{eq:short_sigma} that \(\sigma_g = \sigma_H = \sqrt{T}\,/\sqrt{(1-c_f)\rho}\).
With~\eqref{eq:short_dec} and our choice of \(T\) in~\eqref{eq:T}, we have \(\sqrt{T} = O(\max(\epsg^{-1}, \epsH^{-3/2}))\). We obtain the asymptotic bound of \(n_{\min}\) by plugging in \(\sigma_g\) and \(\sigma_H\).
\end{proof}

\begin{remark}
    When the conditions~\eqref{eq:bounded_noise} do not hold, the algorithm could fail to converge to a \(((1 + c_{1}) \epsg, (1+c) \epsH)\)-2S.\@
    First, the noise in the perturbed gradient and Hessian can be so large that the step is not a descent direction. Second, due to the noise, we may terminate early or fail to terminate timely when checking the approximate second-order conditions. When we terminate early and the noise is not excessive, the solution can still be acceptable since the noisy evaluations satisfy the termination conditions.
\end{remark}
\subsection{Line Search Algorithm}
Instead of using a conservative fixed step size, we can do a line search using backtracking.
The backtracking line search requires an initial value \(\gamma_0\), a decrease parameter \(\beta \in (0,1)\) for the step size, and constants \(c_g \in (0,1-c_1)\), \(c_{H} \in (0, 1-c-\sqrt{\frac{8}{3} c_{2}})\) that determine the amount of decrease we need.
Each line search tries in succession the values \(\gamma_0, \beta \gamma_0, \beta^2 \gamma_0, \dotsc\), until a value is found that satisfies the sufficient decrease condition.
For gradient steps, the condition is
\[
    f(w-\gamma \tilde{g}) \le f(w) - c_g\gamma \|\tilde{g}\|^2,
    \tag{SD1} \label{SD1}
\]
while for negative curvature steps it is
\[
    f(w+\gamma \tilde p) \le f(w)-\frac{1}{2} c_{H} \gamma^{2}|\tilde{\lambda}|.
    \tag{SD2} \label{SD2}
\]
To make line search differentially private, we use the sparse vector technique (SVT) from \citep{dworkAlgorithmicFoundationsDifferential2014}.
We define queries according to~\eqref{SD1} and~\eqref{SD2}:
\begin{subequations}\label{eq:defq}
    \begin{align}
        \label{eq:defqg}
        q_g(\gamma, w) & = f(w) - f(w-\gamma \tilde{g}) - c_g\gamma \|\tilde{g}\|^2,                    \\
        \label{eq:defqH}
        q_H(\gamma, w) & = f(w) - f(w+\gamma \tilde p) - \frac{1}{2} c_{H} \gamma^{2}|\tilde{\lambda}|,
    \end{align}
\end{subequations}
whose nonnegativity corresponds to each of the sufficient decrease conditions.

Algorithm~\ref{svt} specifies the differentially private line search algorithm using SVT, which is adapted from \texttt{AboveThreshold} algorithm \citep{dworkAlgorithmicFoundationsDifferential2014}.

By satisfying the sufficient decrease condition, we try to get a more substantial improvement in the function value than that for the short step algorithm.
As a fallback strategy, we use step sizes similar to the short step values (differing only by a constant factor) if the line search fails, yielding a similar decrease to the short-step case.
We state the complete algorithm enhanced with line search in Algorithm~\ref{alg:dpopt_ls}.
In the algorithm, we compute the fall back step size \(\bar \gamma\) and use a multiplier \(b\) (\(b > 1\)) of them as the initial step size \(b \bar \gamma\) for the line search.
We compute the query sensitivity \(\Delta_q\) accordingly and call the private line search subroutine to find a step size \(\gamma\) that satisfies the sufficient decrease conditions.


\begin{algorithm}[!htb]
    \caption{Private backtracking line search using SVT}\label{svt}
    \begin{algorithmic}
        \State \textbf{Given:} query \(q\) and its sensitivity \(\Delta_q\), initial step size multiplier \(b\), fall back step size \(\bar \gamma\), decrease parameter \(\beta\), privacy parameter \(\lambda\)
        \vspace{0.1in}
        \State \textbf{function} DP-LineSearch(\(q, \Delta_q, \gamma^{\text{init}}, \bar \gamma, \beta, \lambda\))
        \State Initialize \(\gamma \leftarrow \gamma^{\text{init}}\).
        \State Sample \(\xi \sim \lap{2\lambda\Delta_q} \)
        \For{\(i=1,2, \ldots, i_{\max} = \lfloor\log_{\beta}\frac{\bar \gamma}{\gamma^{\text{init}}}\rfloor + 1\)}
        \State Sample \( \nu_i \sim \lap{4\lambda\Delta_q}\)
        \State Evaluate \(q_i = q(\gamma)\)  and \(\tilde q_i = q_i + \nu_i\)
        \If{\(\tilde q_i \ge \xi\)}
        \State HALT and output \(\gamma\)
        \Else
        \State \(\gamma \gets \beta \gamma\)
        \EndIf
        \EndFor
        \State HALT and output \(\bar \gamma\).
    \end{algorithmic}
\end{algorithm}

\begin{algorithm}[!htb]
    \caption{DP Optimization algorithm with Second-Order Guarantees and Backtracking Line Search}\label{alg:dpopt_ls}
    \begin{algorithmic}
        \State \textbf{Given:} noise bound parameters \(c_1\), \(c_2\), \(c\), sufficient decrease parameters \(c_g\), \(c_H\) and initial step size multipliers \(b_g\), \(b_H\), line search decreasing parameters \(\beta_g, \beta_H\), tolerances \(\epsilon_g\) and \(\epsilon_H\), noise parameters \(\sigma_f\), \(\sigma_g\), \(\sigma_H\) and \(\lambda_{SVT}\)
        \vspace{0.1in}
        \State Initialize \(w_0\), sample \(z \sim \mathcal{N}(0, \Delta_f^2 \sigma_f^2)\) and compute MIN\_DEC according to~\eqref{eq:ls_decrease}
        \State Compute an upper bound of the required number of iterations
        \(
        T = \left\lceil\frac{f(w_0) + |z| - \underline{f}}{\text{MIN\_DEC}}\right\rceil
        \)
        \State \(\bar \gamma_g \gets 2\left(1-c_{1}-c_{g}\right)/G\)
        \For{\(k = 0,1, \ldots, T-1\)}
        \State Sample \(\varepsilon_k \sim \mathcal{N}\left(0, \Delta_g^2 \sigma_g^2 I_d\right)\)
        \State Compute the perturbed gradient \(\tilde g_{k} = g_{k}+\varepsilon_{k}\)
        \If{\(\left\|\tilde g_k\right\| > \epsilon_{g}\)}
            \State Define \(q_{k, g}(\gamma) = f(w_k) - f(w_k - \gamma \tilde g_k) - c_g\gamma \|\tilde{g}_k\|^2\)
            \State \(\gamma_{k, g}^{\text{init}} \gets b_g \bar \gamma_g\), \(\; \Delta_{q_{k,g}} \gets \frac{2}{n}\gamma_{k, g}^{\text{init}} B_g \|\tilde{g}_k\|\) \Comment{Line search query sensitivity}
            \State \(\gamma_{k,g} \gets \textsc{DP-LineSearch}(q_{k, g}, \Delta_{q_{k, g}}, \gamma_{k, g}^{\text{init}}, \bar \gamma_g, \beta_g, \lambda_{SVT})\) \Comment{Backtracking line search}
            \State \(w_{k+1} \gets w_k -  \gamma_{k, g} \tilde g_{k}\) \Comment{Gradient step}
        \Else
            \State Sample \(E_k\) such that \(E_k\) is a \(d \times d\) symmetric matrix in which
            each entry on and above its diagonal is i.i.d.\ as \(\mathcal{N}\left(0, \Delta_H^2 \sigma_H^2\right)\)
            \State Compute perturbed Hessian \(\tilde H_k = H_k + E_k\)
            \State Compute the minimum eigenvalue of \(\tilde H_k\) and the corresponding eigenvector \((\tilde \lambda_k, \tilde p_k)\) satisfying~\eqref{eq:min_eig}
            \If{\(\tilde \lambda_{k} < -\epsH\)}
                \State Define \(q_{k, H}(\gamma) = f(w_k) - f(w_k + \gamma \tilde p_k) - \frac{1}{2} c_{H} \gamma^{2}|\tilde{\lambda}_k|\)
                \State \(\bar \gamma_{k,H} \gets t_2|\tilde{\lambda}_k| / {M}\), \(\;\gamma_{k,H}^{\text{init}} \gets b_H \bar \gamma_{k,H}\), \(\; \Delta_{q_{k,H}} \gets \frac{2}{n}\gamma_{k, H}^{\text{init}} B_g\)
                \State 
                \(\gamma_{k,H} \gets \textsc{DP-LineSearch}(q_{H}, \Delta_{q_{k_{H}}}, \gamma_{k,H}^{\text{init}}, \bar \gamma_{k,H}, \beta_H, \lambda_{SVT})\) \Comment{Backtracking line search}
                \State \(w_{k+1} \gets w_k + \gamma_{k, H} \tilde p_k\) \Comment{Negative curvature step}
                \Else
                \State {\bfseries return} \(w_k\)
            \EndIf
        \EndIf
        \EndFor
    \end{algorithmic}
\end{algorithm}

We have the following privacy guarantees.
\begin{theorem}
    Suppose that \(\sigma_f\), \(\sigma_g\), \(\sigma_H\), and \(\lambda\) are given.
    Suppose an actual run of the line search algorithm takes \(k_g\) gradient steps and \(k_H\) negative curvature steps. The run is \(\rho\)-\(z\)CDP where
    \(
    \rho = \frac 12 \left(\frac{1}{\sigma_f^2} +
    \frac{k_g + k_H}{\sigma_g^2} + \frac{k_H}{\sigma_H^2} + \frac{k_g + k_H}{\lambda^2}
    \right).
    \)

    Recall that \(T\) is the maximum number of iterations defined in~\eqref{eq:T}. Let
    \(
    \bar \rho = \frac 12 \left(\frac{1}{\sigma_f^2} +
    \frac{T}{\sigma_g^2} + \frac{T}{\sigma_H^2} + \frac{T}{\lambda^2}
    \right).
    \)
    We always have \(\bar \rho \ge \rho\), so the algorithm is \(\bar \rho\)-\(z\)CDP.\@
    Conversely, for given \(\rho > 0\) and \(\rho_f \in (0, \rho)\), we can choose
    \begin{equation}
        \sigma_f^2 = \frac{1}{2\rho_f}, \quad \sigma_g^2 = \sigma_H^2 = \lambda^2 = \frac{3T}{2(\rho - \rho_f)}, \label{eq:ls_sigma}
    \end{equation}
    to ensure that algorithm is \(\rho\)-\(z\)CDP.\@
\end{theorem}
\begin{proof}
    We know that SVT is \((1 / \lambda)\)-DP.\@ Thus, it satisfies \((1 / (2\lambda^2))\)-\(z\)CDP.\@ The result follows directly from the \(z\)CDP guarantee for the Gaussian mechanism combined with postprocessing and composition of \(z\)CDP.\@
\end{proof}


We now discuss the guarantee of the output of the algorithm. We first derive necessary conditions for sufficient decrease.
\begin{lemma}\label{lem:sd}
    Assume the same bounded noise conditions~\eqref{eq:bounded_noise} as before. With the choice of sufficient decrease coefficients \(c_g \in (0, 1 - c_1), c_{H} \in (0, 1-c-\sqrt{\frac{8}{3} c_{2}})\), let \(\bar \gamma_g = 2\left(1-c_{1}-c_{g}\right)/G \) and
    \(\bar \gamma_H = t_2|\tilde{\lambda}| / {M}\) as defined in Algorithm~\ref{alg:dpopt_ls}, the sufficient decrease conditions~\eqref{SD1} and~\eqref{SD2} are satisfied when \(\gamma \le \bar \gamma_g\) and \(\gamma \in [(t_1 / t_2) \bar \gamma_H, \bar \gamma_H]\), respectively,
    where \(0 < t_1 <  t_2\) are solutions to the following quadratic equation (given our choice of \(c, c_2, c_H\), real solutions exist),
    \[
        r(t) := -\frac{1}{6} t^{2}+\frac{1}{2}\left(1-c-c_{H}\right) t-c_{2} = 0,
    \]
    Explicitly, we have
    \begin{equation} \label{eq:t2}
        t_{1}, t_{2}=\frac{3}{2}\left(1-c-c_{H}\right) \pm 3 \sqrt{\frac{1}{4}(1-c - c_H)^{2}-\frac{2}{3} c_{2}}.
    \end{equation}
    In particular, we have \(q_g(\bar \gamma_g) \ge 0\) and \(q_H(\bar \gamma_H) \ge 0\).
\end{lemma}

We now derive the minimum amount of decrease for each iteration.
\begin{lemma}\label{lem:ls_decrease}
    Using DP line search Algorithm~\ref{alg:dpopt_ls}, assume the same bounded noise conditions~\eqref{eq:bounded_noise} as before. With the choice of sufficient decrease coefficients \(c_g \in (0, 1 - c_1), c_{H} \in (0, 1-c-\sqrt{\frac{8}{3} c_{2}})\), define \(\bar \gamma_g\) and \(\bar \gamma_H\) as before. Choose initial step size multipliers \( b_g, b_H > 1 \)
    and decrease parameters \(\beta_g \in (0,1), \beta_H \in (t_1/t_2, 1) \).
    Let \(i_{\max} = \lfloor\log_{\beta}\max (b_g, b_H)\rfloor + 1\). If \(n\) is at least
    \begin{equation} \label{eq:ls_cond}
        16\lambda \left(\log i_{\max}  + \log \frac{T}{\xi}\right) B_g\max \left(
         \frac{2 b_g}{c_g\epsg},
         \frac{4 b_H M}{t_2 c_{H} \epsH^2}
        \right),
    \end{equation}
    with probability at least \(1- \xi / T\),
    the amount of decrease in a single step is at least
    \begin{equation}\label{eq:ls_decrease}
        \textnormal{MIN\_DEC} =
        \min \left(
        \frac1G (1-c_1-c_g) c_g \epsg^2,
        \frac14 c_{H} t_2^2 \frac{\epsH^{3}}{M^{2}}
        \right).
    \end{equation}
\end{lemma}

With the results above, we can now analyze the guarantees of the line search algorithm under ERM settings.

\begin{theorem}[Sample complexity of the line search algorithm]\label{thm:ls_thm}
    Assuming the same conditions as in the previous lemma, with probability at least \(\{(1-\frac{\zeta}{T})(1 - C\exp{(-C_1Cd)})(1-\xi / T)\}^{T}\), suppose the number of samples \(n\) satisfies \(n \ge n_{\min}\), where
    \begin{equation}\label{min_n_ls}
        n_{\min}  := \max \Biggl(
        \frac{\sqrt{2d} B_g \sigma_g \log \frac{T}{\zeta}}{\min \left(c_{1} \epsg, \, \frac{c_{2}}{M} \epsH^{2}\right)}, \,
        \frac{C\sqrt{d} B_H \sigma_H \log \frac{T}{\zeta}}{c \, \epsH},
        16\lambda \left(\log i_{\max} + \log \frac{T}{\xi} \right) B_g \max \left(
             \frac{2 b_g}{c_g\epsg},
            \frac{4b_H M}{t_2 c_{H} \epsH^2}
            \right)
        \Biggr).
    \end{equation}
    The output of the algorithm is a \(((1+c_{1}) \epsg, (1+c) \epsH)\)-2S.\@ With the choice of \(\sigma\)'s and \(\lambda\) in~\eqref{eq:ls_sigma}, hiding logarithmic terms and constants, the asymptotic dependence of \(n_{\min}\) on \((\epsg, \epsH)\) and \(\rho\), is
    \begin{equation}
        n_{\min} = \frac{\sqrt{d}}{\sqrt{\rho}} \tilde O \left(
        \max \left(
            \epsg^{-2}, \epsg^{-1} \epsH^{-2}, \epsH^{-7/2}
            \right)
        \right).
    \end{equation}
\end{theorem}
When \((\epsg,\epsH)  = (\alpha,\sqrt{M \alpha})\), the dependence simplifies to \(\frac{\sqrt{d}}{\sqrt{\rho}}\, \tilde O (\alpha^{-2})\).
\begin{proof}
    The proof is similar to Theorem~\ref{thm:short_thm} using Lemma~\ref{eq:ls_decrease}. We have an additional term in our success probability due to the SVT line search step.  For the asymptotic bound of \(n_{\min}\), we note that MIN\_DEC~\eqref{eq:ls_decrease}, \(T\), \(\sigma_g\) and \(\sigma_H\) are the same as those of the short step algorithm, up to a constant. The additional requirement~\eqref{eq:ls_cond} for \(n\) is
    \(O\left(\frac{\lambda \log T}{\max(\epsg^{-1}, \epsH^{-2})}\right)\). Since we choose \(\lambda = \sigma_g\), it is in the same order as the first term inside the \(\max\) expression of \(n_{\min}\) in~\eqref{min_n_ls}. Thus, the asymptotic bound of \(n_{\min}\) is the same as that of the short step algorithm.
\end{proof}

\subsection{Mini-batching}

Mini-batching is an effective method for speeding up the algorithm. By sampling a subset of data points from the full dataset, we can compute the average risk over the mini-batch and modify the short step algorithm to evaluate gradients and Hessians over the subset.

Formally, in each iteration \(k\), we sample \(m\) data points from \(D\) without replacement, forming the mini-batch \(S_k\). The objective is now the average risk over set \(S_k\), that is,
\[
    f_{S_k} := \frac{1}{m} \sum_{i \in S_k} \ell\left(w_k, x_{i}\right).
\]
We show that the sample complexity of the mini-batch version of the algorithm remains \(\tilde O (\frac{\sqrt{d\ln(1/\delta)}}{\varepsilon \alpha^2})\) when \((\epsg,\epsH)  = (\alpha,\sqrt{M \alpha})\) for \((\varepsilon, \delta)\)-DP,
matching the sample complexity of the mini-batch version of DP-TR.\@ The details are in Appendix~\ref{sec:subsample_analysis}.

\subsection{Discussion: Two-phase Strategy and Eigenvalue Computation}\label{sec:improvement}
To speed up the algorithm, we propose a two-phase strategy that addresses the issue of the pessimistic estimate of \(T\) in~\eqref{eq:T}, which is based on MIN\_DEC obtained from the worst-case analysis.

The first phase involves using a fraction of the privacy budget (3/4, say) to try out a smaller value of \(T\). Using a smaller value of \(T\) results in less noise and potentially faster convergence. If we are unable to find a desired solution, we then move on to a second phase.
In this phase, we fall back to the original method for estimating \(T\) using the remaining privacy budget. Using the last iterate as a warm start can help improve the efficiency of this phase.

Additionally, we can use the Lanczos method to find an approximation to the minimum eigenvalue and eigenvector, in place of a direct eigenvalue computation. This alternative yields a slightly different analysis. See Appendix~\ref{sec:lanczos} for a discussion.






\section{Experiments}\label{sec:exp}


We carry out numerical experiments to demonstrate the performance of our DP optimization algorithms, following similar experimental protocols to \citep{wangEscapingSaddlePoints2021}.
We use Covertype dataset and perform necessary data pre-processing. Details of the dataset and additional experiments can be found in Appendix~\ref{sec:exp_details}. 

Let \(x_i\) be the feature vector and \(y_i \in \{-1, +1\}\) be the binary label. We investigate the nonconvex ERM loss\footnote{Upon checking, the loss has Lipschitz gradients and Hessians as long as the feature vector \(x_i\)'s are bounded.}:
\[
    \min _{w \in \mathbb{R}^{p}} \frac{1}{n} \sum_{i=1}^{n} \log \left(1+\exp \left(-y_{i}\left\langle x_{i}, w\right\rangle\right)\right)+\lambda r(w),
\]
where \(r(w) = \sum_{i=1}^p \frac{w_i^2}{1+ w_i^2}\)  is the nonconvex regularizer. In our experiments, we choose \(\lambda = 10^{-3}\).

We compare our algorithms with DP-TR.\@
\footnote{We tried to implement DP-GD, but could not produce practical results using the algorithmic parameters described in the DP-GD paper.}
To ensure consistency in the comparison of results, we have modified the DP-TR method to include an explicit check for approximate second-order  conditions, similar to the approach employed by our algorithms. This modification allows DP-TR to terminate when these conditions are satisfied,

We run the experiment under two settings:
\begin{enumerate}
    \item Finding a loose solution, \(\epsg = 0.060\) and \(\epsH \approx 0.245\). In this setting, our requirement for the 2S is loose. This translates to a large sample size \(n\) compared to the required sample complexity.
    \item Finding a tight solution: \(\epsg = 0.030\) and \(\epsH \approx 0.173\). In this setting, our requirement for the 2S is tight. We have a small sample size \(n\) compared to the required sample complexity.
\end{enumerate}
 For each setting, we pick different levels of privacy budget \(\varepsilon\) and run each configuration with five different random seeds. We convert differential privacy schemes to \((\varepsilon, \delta)\)-DP when necessary for the comparison.
We present the aggregated results in the tables below. In each entry, we report the mean \(\pm\) standard deviation of the values across five runs. If any of the five runs failed to find a solution, or it found a solution but failed to terminate due to the noise, we report the runtime\footnote{The runtime here is expressed in a unit determined by the Python function \texttt{time.perf\_counter()}} with \(\times\).


In the table, we use acronyms for methods: TR for DP-TR, OPT for our proposed algorithms and 2OPT for their two-phase variants, OPT-LS for our proposed algorithms with line search, and the ones with ``-B'' use mini-batching.



\begin{table*}[!htb]
    \centering
    \caption{Covertype: finding a loose solution, \((\epsg, \epsH) = (0.060, 0.245)\)}
\begin{tabular}{ccccccc}
    \toprule
    \multirow{2}{*}{method} & \multicolumn{2}{c}{\(\varepsilon = 0.2\)}                                     & \multicolumn{2}{c}{\(\varepsilon = 0.6\)}  & \multicolumn{2}{c}{\(\varepsilon = 1.0\)}            \\
    \cmidrule(lr){2-3}
    \cmidrule(lr){4-5}
    \cmidrule(lr){6-7}
    & final loss & runtime & loss & runtime & loss & runtime \\
    \midrule
    TR      &  \(0.729 \pm 0.028\) & \(10.1 \pm 9.9\) &  \(0.729 \pm 0.026\) & \(8.3 \pm 8.6\) &  \(0.729 \pm 0.026\) & \(9.5 \pm 9.1\) \\
    TR-B    &   \(0.729 \pm 0.029\) & \(2.2 \pm 2.0\) &  \(0.728 \pm 0.027\) & \(2.2 \pm 2.4\) &  \(0.729 \pm 0.028\) & \(2.5 \pm 2.4\) \\
    \midrule
    OPT     &      \(0.581 \pm 0.057\) & \( \times \) &  \(0.712 \pm 0.018\) & \(0.6 \pm 0.2\) &  \(0.712 \pm 0.017\) & \(0.5 \pm 0.2\) \\
    OPT-B   &   \(0.712 \pm 0.018\) & \(3.1 \pm 2.9\) &  \(0.712 \pm 0.018\) & \(3.2 \pm 3.0\) &  \(0.712 \pm 0.018\) & \(2.9 \pm 2.9\) \\
    OPT-LS  &      \(0.577 \pm 0.032\) & \( \times \) &  \(0.687 \pm 0.028\) & \(\mathbf{0.4 \pm 0.1}\) &  \(0.699 \pm 0.018\) & \(\mathbf{0.4 \pm 0.1}\) \\
    \midrule
    2OPT    &      \(0.626 \pm 0.078\) & \( \times \) &  \(0.712 \pm 0.017\) & \(0.6 \pm 0.2\) &  \(0.712 \pm 0.018\) & \(0.6 \pm 0.2\) \\
    2OPT-B  &   \(0.712 \pm 0.018\) & \(1.4 \pm 0.3\) &  \(0.712 \pm 0.018\) & \(1.4 \pm 0.4\) &  \(0.712 \pm 0.018\) & \(2.0 \pm 1.7\) \\
    2OPT-LS &   \(0.699 \pm 0.018\) & \(\mathbf{0.5 \pm 0.2}\) &  \(0.699 \pm 0.018\) & \(0.5 \pm 0.2\) &  \(0.699 \pm 0.018\) & \(0.5 \pm 0.2\) \\
    \bottomrule
\end{tabular}
\end{table*}

\begin{table*}[!htb]
    \centering
    \caption{Covertype: finding a tight solution: \((\epsg, \epsH) = (0.030, 0.173)\)}
    \begin{tabular}{ccccccc}
        \toprule
        \multirow{2}{*}{method} & \multicolumn{2}{c}{\(\varepsilon = 0.2\)}                                     & \multicolumn{2}{c}{\(\varepsilon = 0.6\)}  & \multicolumn{2}{c}{\(\varepsilon = 1.0\)}            \\
        \cmidrule(lr){2-3}
        \cmidrule(lr){4-5}
        \cmidrule(lr){6-7}
        & final loss & runtime & loss & runtime & loss & runtime \\ \midrule
    TR      &     \(0.516 \pm 0.005\) & \( \times \) &  \(0.607 \pm 0.007\) & \(99.6 \pm 32.2\) &  \(0.607 \pm 0.005\) & \(90.8 \pm 21.6\) \\
    TR-B    &     \(0.517 \pm 0.005\) & \( \times \) &   \(0.603 \pm 0.005\) & \(32.6 \pm 7.9\) &  \(0.607 \pm 0.003\) & \(33.4 \pm 14.4\) \\
    \midrule
    OPT     &     \(0.506 \pm 0.001\) & \( \times \) &       \(0.535 \pm 0.015\) & \( \times \) &    \(0.592 \pm 0.003\) & \(1.8 \pm 0.5\) \\
    OPT-B   &  \(0.597 \pm 0.003\) & \(\mathbf{1.3 \pm 0.3}\) &    \(0.597 \pm 0.003\) & \(1.3 \pm 0.2\) &    \(0.597 \pm 0.003\) & \(1.4 \pm 0.3\) \\
    OPT-LS  &     \(0.525 \pm 0.009\) & \( \times \) &       \(0.527 \pm 0.009\) & \( \times \) &       \(0.549 \pm 0.006\) & \( \times \) \\
    \midrule
    2OPT    &     \(0.502 \pm 0.001\) & \( \times \) &       \(0.513 \pm 0.003\) & \( \times \) &       \(0.519 \pm 0.003\) & \( \times \) \\
    2OPT-B  &  \(0.597 \pm 0.003\) & \(2.1 \pm 0.4\) &    \(0.597 \pm 0.003\) & \(2.3 \pm 0.5\) &    \(0.597 \pm 0.003\) & \(2.3 \pm 0.6\) \\
    2OPT-LS &  \(0.577 \pm 0.008\) & \(2.1 \pm 1.0\) &    \(0.591 \pm 0.001\) & \(\mathbf{0.6 \pm 0.1}\) &    \(0.591 \pm 0.001\) & \(\mathbf{0.8 \pm 0.2}\) \\
        \bottomrule
    \end{tabular}
\end{table*}

Experimental results show that for finding a loose solution under high privacy budgets \(\varepsilon = 0.6, 1.0\), our short step algorithm OPT outperforms TR, with much less runtime and lower final loss.
Under the low privacy budget \(\varepsilon = 0.2\), although OPT can fail to terminate with success, we see that the final loss is even lower than TR.\@ The reason is as follows, due to the conservative estimate of the decrease, the per iteration privacy budget is low, so we cannot check 2S conditions accurately enough due to the noise. In practice, we can stop early and the solution is still acceptable despite the failure of the termination. Heuristics may be employed to spend extra privacy budget to check 2S conditions.

Line search and mini-batching improve upon the short step algorithm, especially when combined with our two-phase strategy. We remark that similar to OPT, OPT-LS has an even more conservative theoretical minimum decrease.
The two-phase strategy, using an aggressive estimate of the decrease, complements line search. We observe that 2OPT-LS performs consistently well across all privacy budget levels and under two settings.
Finally, we remark that the number of Hessian evaluations is minimal. See Appendix~\ref{sec:add_exp} for details.



\section{Conclusion}
We develop simple differentially private optimization algorithms based on an elementary algorithm for finding an approximate second-order optimal point of a smooth nonconvex function.
The proposed algorithms take  noisy gradient steps or negative curvature steps based on a noisy Hessian on nonconvex ERM problems.
To obtain a method that is more practical than conservative short-step methods, we employ line searches, mini-batching, and a two-phase strategy. 
We track privacy leakage using \(z\)CDP (RDP for mini-batching).
Our work matches the sample complexity of DP-GD, but with a much simpler analysis.
Although DP-TR has a better sample complexity, its mini-batched version has the same complexity as ours.
Our algorithms have a significant advantage over DP-TR in terms of runtime. 2OPT-LS, which combines the line search and the two-phase strategy, consistently outperform DP-TR in numerical experiments.

\bibliographystyle{abbrvnat}
\bibliography{ref}

\begin{thebibliography}{17}
\providecommand{\natexlab}[1]{#1}
\providecommand{\url}[1]{\texttt{#1}}
\expandafter\ifx\csname urlstyle\endcsname\relax
  \providecommand{\doi}[1]{doi: #1}\else
  \providecommand{\doi}{doi: \begingroup \urlstyle{rm}\Url}\fi

\bibitem[Balle et~al.(2018)Balle, Barthe, and
  Gaboardi]{ballePrivacyAmplificationSubsampling2018}
B.~Balle, G.~Barthe, and M.~Gaboardi.
\newblock Privacy {{Amplification}} by {{Subsampling}}: {{Tight Analyses}} via
  {{Couplings}} and {{Divergences}}, Nov. 2018.

\bibitem[Bun and Steinke(2016)]{bunConcentratedDifferentialPrivacy2016}
M.~Bun and T.~Steinke.
\newblock Concentrated {{Differential Privacy}}: {{Simplifications}},
  {{Extensions}}, and {{Lower Bounds}}.
\newblock \emph{arXiv:1605.02065 [cs, math]}, May 2016.

\bibitem[Carmon et~al.(2017)Carmon, Duchi, Hinder, and
  Sidford]{carmonAcceleratedMethodsNonConvex2017}
Y.~Carmon, J.~C. Duchi, O.~Hinder, and A.~Sidford.
\newblock Accelerated {{Methods}} for {{Non-Convex Optimization}}, Feb. 2017.

\bibitem[Chen and Lee(2020)]{chenStochasticAdaptiveLine2020}
C.~Chen and J.~Lee.
\newblock Stochastic {{Adaptive Line Search}} for {{Differentially Private
  Optimization}}.
\newblock \emph{arXiv:2008.07978 [cs, stat]}, Aug. 2020.

\bibitem[Dwork and Roth(2014)]{dworkAlgorithmicFoundationsDifferential2014}
C.~Dwork and A.~Roth.
\newblock The {{Algorithmic Foundations}} of {{Differential Privacy}}.
\newblock \emph{Foundations and Trends\textregistered{} in Theoretical Computer
  Science}, 9\penalty0 (3\textendash 4):\penalty0 211--407, Aug. 2014.
\newblock ISSN 1551-305X.
\newblock \doi{10.1561/0400000042}.

\bibitem[Kohler and Lucchi(2017)]{kohlerSubsampledCubicRegularization2017}
J.~M. Kohler and A.~Lucchi.
\newblock Sub-sampled {{Cubic Regularization}} for {{Non-convex Optimization}}.
\newblock \emph{arXiv:1705.05933 [cs, math, stat]}, July 2017.

\bibitem[Mironov(2017)]{mironovRenyiDifferentialPrivacy2017}
I.~Mironov.
\newblock Renyi {{Differential Privacy}}.
\newblock \emph{2017 IEEE 30th Computer Security Foundations Symposium (CSF)},
  pages 263--275, Aug. 2017.
\newblock \doi{10.1109/CSF.2017.11}.

\bibitem[{Shalev-Shwartz} and
  {Ben-David}(2014)]{shalev-shwartzUnderstandingMachineLearning2014}
S.~{Shalev-Shwartz} and S.~{Ben-David}.
\newblock \emph{Understanding {{Machine Learning}}: {{From Theory}} to
  {{Algorithms}}}.
\newblock {Cambridge University Press}, {Cambridge}, 2014.
\newblock ISBN 978-1-107-05713-5.
\newblock \doi{10.1017/CBO9781107298019}.

\bibitem[Tao(2012)]{taoTopicsRandomMatrix2012}
T.~Tao.
\newblock \emph{Topics in {{Random Matrix Theory}}}.
\newblock {American Mathematical Soc.}, Mar. 2012.
\newblock ISBN 978-0-8218-7430-1.

\bibitem[Vershynin(2018)]{vershyninHighDimensionalProbabilityIntroduction2018}
R.~Vershynin.
\newblock \emph{High-{{Dimensional Probability}}: {{An Introduction}} with
  {{Applications}} in {{Data Science}}}.
\newblock Cambridge {{Series}} in {{Statistical}} and {{Probabilistic
  Mathematics}}. {Cambridge University Press}, {Cambridge}, 2018.
\newblock ISBN 978-1-108-41519-4.
\newblock \doi{10.1017/9781108231596}.

\bibitem[Wang and Xu(2019)]{wangDifferentiallyPrivateEmpirical2019a}
D.~Wang and J.~Xu.
\newblock Differentially {{Private Empirical Risk Minimization}} with {{Smooth
  Non-Convex Loss Functions}}: {{A Non-Stationary View}}.
\newblock \emph{Proceedings of the AAAI Conference on Artificial Intelligence},
  33\penalty0 (01):\penalty0 1182--1189, July 2019.
\newblock ISSN 2374-3468.
\newblock \doi{10.1609/aaai.v33i01.33011182}.

\bibitem[Wang and Xu(2021)]{wangEscapingSaddlePoints2021}
D.~Wang and J.~Xu.
\newblock Escaping {{Saddle Points}} of {{Empirical Risk Privately}} and
  {{Scalably}} via {{DP-Trust Region Method}}.
\newblock In F.~Hutter, K.~Kersting, J.~Lijffijt, and I.~Valera, editors,
  \emph{Machine {{Learning}} and {{Knowledge Discovery}} in {{Databases}}},
  Lecture {{Notes}} in {{Computer Science}}, pages 90--106, {Cham}, 2021.
  {Springer International Publishing}.
\newblock ISBN 978-3-030-67664-3.
\newblock \doi{10.1007/978-3-030-67664-3_6}.

\bibitem[Wang et~al.(2018{\natexlab{a}})Wang, Ye, and
  Xu]{wangDifferentiallyPrivateEmpirical2018}
D.~Wang, M.~Ye, and J.~Xu.
\newblock Differentially {{Private Empirical Risk Minimization Revisited}}:
  {{Faster}} and {{More General}}.
\newblock \emph{arXiv:1802.05251 [cs, stat]}, Feb. 2018{\natexlab{a}}.

\bibitem[Wang et~al.(2019)Wang, Chen, and
  Xu]{wangDifferentiallyPrivateEmpirical2019}
D.~Wang, C.~Chen, and J.~Xu.
\newblock Differentially {{Private Empirical Risk Minimization}} with
  {{Non-convex Loss Functions}}.
\newblock In \emph{Proceedings of the 36th {{International Conference}} on
  {{Machine Learning}}}, pages 6526--6535. {PMLR}, May 2019.

\bibitem[Wang et~al.(2018{\natexlab{b}})Wang, Balle, and
  Kasiviswanathan]{wangSubsampledEnyiDifferential2018}
Y.-X. Wang, B.~Balle, and S.~Kasiviswanathan.
\newblock Subsampled {{R}}\textbackslash 'enyi {{Differential Privacy}} and
  {{Analytical Moments Accountant}}.
\newblock \emph{arXiv:1808.00087 [cs, stat]}, Dec. 2018{\natexlab{b}}.

\bibitem[Wright and Recht(2022)]{wrightOptimizationDataAnalysis2022}
S.~J. Wright and B.~Recht.
\newblock \emph{Optimization for {{Data Analysis}}}.
\newblock {Cambridge University Press}, {Cambridge}, 2022.
\newblock ISBN 978-1-316-51898-4.
\newblock \doi{10.1017/9781009004282}.

\bibitem[Zhang et~al.(2017)Zhang, Zheng, Mou, and
  Wang]{zhangEfficientPrivateERM2017}
J.~Zhang, K.~Zheng, W.~Mou, and L.~Wang.
\newblock Efficient {{Private ERM}} for {{Smooth Objectives}}.
\newblock In \emph{Proceedings of the {{Twenty-Sixth International Joint
  Conference}} on {{Artificial Intelligence}}}, pages 3922--3928, {Melbourne,
  Australia}, Aug. 2017. {International Joint Conferences on Artificial
  Intelligence Organization}.
\newblock ISBN 978-0-9992411-0-3.
\newblock \doi{10.24963/ijcai.2017/548}.

\end{thebibliography}

\newpage
\appendix
\section{Brief Review of Differential Privacy}\label{sec:prelim-dp}
\begin{definition}[\((\varepsilon, \delta)\)-DP \citep{dworkAlgorithmicFoundationsDifferential2014}]
    A randomized algorithm \(\A\) is \((\varepsilon, \delta)\)-DP if for all neighboring datasets \(D, D'\) and for all events \(S\) in the output space of \(\A\), the following holds:
    \[
        \Pr\left(\A(D) \in S\right) \le
        e^\varepsilon \Pr\left(\A(D') \in S\right) + \delta.
    \]
    When \(\delta = 0\), we say \(\A\) is \(\varepsilon\)-DP.\@ In the ERM setting, we say \(D'\) is a neighboring dataset of \(D\) if they differ on just one data point, that is, by changing some data point \(x_k\) in \(D\) to \(x_k'\), we obtain dataset \(D'\).\@
\end{definition}

Rényi-DP (RDP) was introduced by Mironov as a relaxation of the original DP.\@
\begin{definition}[Rényi divergence]
    For two probability distributions \(P\) and \(Q\) defined over \(\mathcal{R}\), the Rényi divergence of order \(\alpha>1\) is
    \[
        D_{\alpha}(P \| Q) = \frac{1}{\alpha-1} \log \mathrm{E}_{w \sim Q}\left(\frac{P(w)}{Q(w)}\right)^{\alpha}.
    \]
\end{definition}
\begin{definition}[\((\alpha, \epsilon)\)-RDP \citep{mironovRenyiDifferentialPrivacy2017}]
    A randomized algorithm \(\M: \mathcal{D} \to \mathcal{R}\) is \((\alpha, \epsilon)\)-RDP if for all neighboring dataset pairs \(D, D'\), the following holds
    \[
        D_{\alpha}(\M(D) \| \M(D')) \le \epsilon.
    \]
\end{definition}
Another notion of differential privacy is Zero-Concentrated Differential Privacy (\(z\)CDP), which requires a linear bound for the divergence of all orders.
\begin{definition}[\(z\)CDP \citep{bunConcentratedDifferentialPrivacy2016}]
    A randomized algorithm \(\M: \mathcal{D} \to \mathcal{R}\) satisfies \((\xi,\rho)\)-\(z\)CDP if for all neighboring dataset pairs \(D, D'\) and all \(\alpha \in (1, \infty)\), the following holds:
    \[
        D_{\alpha}(\M(D) \| \M(D')) \le \xi + \rho \alpha.
    \]
    Equivalently, a randomized algorithm \(\M\) satisfies \(\rho\)-\(z\)CDP if for all \(\alpha \in (1, \infty)\), \(\M\) satisfies \((\alpha, \xi + \rho \alpha)\)-RDP.\@
    If this definition holds for \(\xi=0\), we use the term \(\rho\)-\(z\)CDP instead.
\end{definition}

RDP and \(z\)CDP have some properties in common \citep{mironovRenyiDifferentialPrivacy2017, bunConcentratedDifferentialPrivacy2016}.

\begin{proposition}[Composition of RDP]
    Suppose that \(\M_1: \mathcal{D} \to \mathcal{R}_{1}\) is \(\left(\alpha, \epsilon_{1}\right)\)-RDP and \(M_2: \mathcal{R}_{1} \times\) \(\mathcal{D} \to \mathcal{R}_{2}\) is \(\left(\alpha, \epsilon_{2}\right)\)-RDP.\@
    Then the mechanism defined as \((X, Y)\), where \(X \sim \M_1(D)\) and \(Y \sim \M_2(X, D)\) is \(\left(\alpha, \epsilon_{1}+\epsilon_{2}\right)\)-RDP.\@
\end{proposition}

\begin{proposition}[Composition of \(z\)CDP]
    Suppose that \(\M_1: \mathcal{D} \to \mathcal{R}_{1}\) is \(\rho_1\)-\(z\)CDP and \(M_2: \mathcal{R}_{1} \times\) \(\mathcal{D} \to \mathcal{R}_{2}\) is \(\rho_2\)-\(z\)CDP.\@ Then the mechanism defined as \((X, Y)\), where \(X \sim \M_1(D)\) and \(Y \sim \M_2(X, D)\) is \((\rho_1+\rho_2)\)-\(z\)CDP.\@
\end{proposition}

\begin{proposition}[Preservation under Postprocessing]
    Consider the mappings \(\M: \mathcal{D} \to \mathcal{R}\)  and \(g: \mathcal{R} \to \mathcal{R}^{\prime}\). It follows from the analog of the data processing inequality that \(D_{\alpha}(P \| Q) \geq D_{\alpha}(g(P) \| g(Q))\). This shows that if \(\M(\cdot)\) is \((\alpha, \epsilon)\)-RDP, so is \(g(\M(\cdot))\). Similarly, if \(\M(\cdot)\) is \(\rho\)-\(z\)CDP, so is \(g(\M(\cdot))\).
\end{proposition}

We can convert easily from one notion of differential privacy to another.
\begin{proposition}[RDP to \((\varepsilon, \delta)\)-DP]
    If \(\M\) is an \(\left(\alpha, \epsilon \right)\)-RDP mechanism, then it is \( \left(\epsilon + \frac{\log 1 / \delta}{\alpha - 1}, \delta\right)\)-DP for any \(0 < \delta < 1\).
\end{proposition}

\begin{proposition}
    [\(\epsilon\)-DP to \(z\)CDP] If \(\M\) is an \(\varepsilon\)-DP mechanism, then it is also \((\frac 12 \varepsilon^2)\)-\(z\)CDP.\@
\end{proposition}

\begin{proposition}[\(z\)CDP to \((\varepsilon, \delta)\)-DP]\label{prop:zcdp2dp}
    Suppose that \(\M: \mathcal{D} \to \mathcal{R}\) is \((\xi, \rho)\)-\(z\)CDP.\@ Then \(M\) is also \((\varepsilon, \delta)\)-DP for all \(\delta>0\) and
    \[
        \varepsilon=\xi+\rho+\sqrt{4 \rho \log (1 / \delta)} .
    \]
    Thus to achieve a \((\varepsilon, \delta)\)-DP guarantee for given \(\varepsilon\) and \(\delta\), it suffices to satisfy \((\xi, \rho)\)-\(z\)CDP with
    \[
        \rho=(\sqrt{\varepsilon-\xi+\log (1 / \delta)}-\sqrt{\log (1 / \delta)})^{2} \approx \frac{(\varepsilon-\xi)^{2}}{4 \log (1 / \delta)}.
    \]
\end{proposition}

A common way to achieve differential privacy is to add Gaussian noise to the output.
\begin{proposition}[Gaussian Mechanism]\label{prop:gaussian-mechanism}
    Given any function \(h: \mathcal{X}^{n} \rightarrow \mathbb{R}^{d}\), the Gaussian Mechanism is defined as:
    \[
        \mathbf{G}_{\sigma}h(D)=h(D) + N(0, \Delta_h^2 \sigma^2 I_d),
    \]
    where \(\Delta_h\) denotes the \(\ell_{2}\)-sensitivity of the function \(h\), defined as
    \begin{equation} \label{eq:sens}
        \Delta_h=\sup _{D \sim D^{\prime}} \| h(D) - h(D^\prime) \|.
    \end{equation}
    The Gaussian Mechanism \(\mathbf{G}_{\sigma}h\) satisfies \((\alpha, \alpha / (2\sigma^2))\)-RDP for all \(\alpha \ge 1\) and thus also satisfies \((1 / 2\sigma^2)\)-\(z\)CDP.\@
\end{proposition}

We defer results for \((\varepsilon, \delta)\)-DP to Appendix~\ref{sec:subsample_analyis_dp}.






\section{Analysis of the Mini-Batch Algorithms}\label{sec:subsample_analysis}
For the line search version of the algorithm, we need additional assumptions if we want to check the sufficient decrease conditions using the mini-batch loss.
For simplicity, we only consider the short version of the algorithm in this section.
\subsection{RDP Analysis}
As mentioned in the main text, in each iteration \(k\), we sample \(m\) data points from \(D\) without replacement, forming the mini-batch \(S_k\). We compute the average risk over set \(S_k\), that is,
\[
    f_{S_k} := \frac{1}{m} \sum_{i \in S_k} \ell\left(w_k, x_{i}\right).
\]
We evaluate the gradient and the Hessian similarly on the mini-batch \(S_k\) of size \(m\), which we will write as \(g_{S_k}\), \(H_{S_k}\) and let \(\tilde g_{S_k}\), \(\tilde H_{S_k}\) be their perturbed versions respectively. The other parts of the algorithm remain unchanged.
The sensitivity \(\Delta_f, \Delta_g, \Delta_H\) as stated in~\eqref{eq:sensitivity} will be scaled accordingly by replacing \(n\) in their denominator by the mini-batch size \(m\). Throughout this section, we use \(s = m/n\) to denote the sampling fraction.

Let \(g_k\) and \(H_k\) be the gradient and the Hessian evaluated on the full dataset \(D\). We can decompose the deviation of their noisy approximation as follows,
\begin{subequations}\label{eq:error_decomp}
    \begin{align}
        \left\|\tilde g_{S_k} - g_k \right\|
         & \le \left\|\tilde g_{S_k} - g_{S_k} \right\|  +
        \left\|g_{S_k} - g_k \right\|,                        \\
        \left\|\tilde H_{S_k} - H_k \right\|
         & \le \left\|\tilde H_{S_k} - H_{S_k} \right\| +
        \left\|H_{S_k} - H_k \right\|,
    \end{align}
\end{subequations}
where the first term in the bound is due to the added Gaussian noise, and the second term is due to subsampling. We will bound two terms separately with high probability.

We have the following subsampling concentration results from \citep{kohlerSubsampledCubicRegularization2017}
\begin{lemma}[Gradient deviation bound]
    We have with probability at least \(1-\eta\) that
    \[
        \left\|g_{S_k} - g_k\right\| \leq 4 \sqrt{2} B_g \sqrt{\frac{\log (2 d / \eta)+1 / 4}{\left|S_{k}\right|}}.
    \]
\end{lemma}

\begin{lemma}[Hessian deviation bound]
    We have with probability at least \(1-\eta\) that
    \[
        \left\|H_{S_k} - H_k\right\|
        \leq 4 B_H \sqrt{\frac{\log (2 d / \eta)}{\left|S_{k}\right|}}.
    \]
\end{lemma}

For the subsampling error~\eqref{eq:error_decomp}, we use Gaussian concentration results~\ref{gauss_concentration} described before to bound the first term, and subsampling results stated above to bound the second term. For iteration \(k\), with probability at least \((1 - \frac{\zeta}{T})(1 - C\exp{(-C_1Cd)})(1-\eta/T)^2\), we have that
\begin{equation}
    \begin{aligned}
        \left\|\tilde g_{S_k} - g_k \right\|
         & \leq \sqrt{2d} \eta_g \sigma_g \log \frac{T}{\zeta} +
        4 \sqrt{2} B_g \sqrt{\frac{\log (2 d T/ \eta)+1 / 4}{\left|S_{k}\right|}}, \\
        \left\|\tilde H_{S_k} - H_k \right\|
         & \leq C \sqrt{d} \eta_H \sigma_H +
        4 B_H \sqrt{\frac{\log (2 d T / \eta)}{\left|S_{k}\right|}}.
    \end{aligned}
\end{equation}

It suffices to require that each term in the right-hand sides of the bound above is bounded by \(1/2\) of the corresponding term in the right-hand sides of~\eqref{eq:bounded_noise}, so that we can use a similar analysis.

For deviation due to subsampling, we need
\begin{subequations}\label{subsample-error-bound}
    \begin{align}
        4 \sqrt{2} B_g \sqrt{\frac{\log (2 d T/ \eta)+1 / 4}{\left|S_{k}\right|}} & \le \frac12
        \min \left(c_{1} \epsg, \frac{c_{2}}{M} \epsH^{2}\right),                                          \\
        4 B_H \sqrt{\frac{\log (2 d T / \eta)}{\left|S_{k}\right|}}               & \le \frac12 c \,\epsH.
    \end{align}
\end{subequations}
Rearranging the terms, we have a condition for the mini-batch size,
\begin{equation}
    \left|S_{k}\right| \ge \max \left(
    64 B_g^2 (\log (2 d T/ \eta)+1 / 4) \max \left(c_{1}^{-2} \epsg^{-2}, \frac{M^2}{c_{2}^2} \epsH^{-4}\right),
    32 B_H^2 \log (2 d T/ \eta) c^{-2} \epsH^{-2}
    \right).
\end{equation}
The following convergence result is immediate, based on the same analysis as in the full-batch case (cf. Theorem~\ref{thm:ls_thm}).
\begin{theorem}\label{thm:ls_batch}
    With probability at least \(\{(1 - \zeta / T)(1 - C\exp{(-C_1Cd)})(1-\eta/T)^2\}^{T}\), suppose the number of samples \(n\) satisfies \(n \ge n_{\min}\), where
    \begin{equation} \label{min_n_subsampled-ls}
        \begin{aligned}
            n_{\min} & :=  \max \Biggl(
            \frac{2\sqrt{2d} B_g \sigma_g \log \frac{T}{\zeta}}{\min \left(c_{1} \epsg, \, \frac{c_{2}}{M} \epsH^{2}\right)}, \,
            \frac{2C\sqrt{d} B_H \sigma_H \log \frac{T}{\zeta}}{c \, \epsH},                                                                          \\
                     & \quad\quad s^{-1} 64 B_g^2 (\log (2 d T/ \eta)+1 / 4) \max \left(c_{1}^{-2} \epsg^{-2}, \frac{M^2}{c_{2}^2} \epsH^{-4}\right),
            s^{-1} 32 B_H^2 \log (2 d T/ \eta) c^{-2} \epsH^{-2}
            \Biggr).
        \end{aligned}
    \end{equation}
    The output of the mini-batch short step algorithm is a \(((1+c_{1}) \epsg, (1+c) \epsH)\)-2S.\@ With the choice of \(\sigma\)'s in~\eqref{eq:short_sigma}, hiding logarithmic terms and constants, the asymptotic dependence of \(n_{\min}\) on \((\epsg, \epsH)\) and \(\rho\), is
    \begin{equation}
        n_{\min} = \frac{\sqrt{d}}{\sqrt{\rho}} \tilde O \left(
        \max \left(
            \epsg^{-2}, \epsg^{-1} \epsH^{-2}, \epsH^{-7/2}
            \right)
        \right).
    \end{equation}
\end{theorem}

We now discuss privacy guarantees. It is impossible to deal with subsampling using \(z\)-CDP, but under RDP, \citet{wangSubsampledEnyiDifferential2018} provides a generalized analysis for subsampling:
\begin{theorem}[RDP for Subsampled Mechanisms]\label{subsampled-rdp}
    Given a dataset of \(n\) points drawn from a domain \(\mathcal{X}\) and a (randomized) mechanism \(\mathcal{M}\) that takes an input from \(\mathcal{X}^{m}\) for \(m \leq n\), let the randomized algorithm \(\mathcal{M} \circ \mathbf{subsample}\) be defined as (1) subsample: subsample without replacement m datapoints of the dataset (with sampling fraction  \(s=m/n\)),
    and (2) apply \(\mathcal{M}\): a randomized algorithm taking the subsampled dataset as the input. For all integers \(\alpha \geq 2\), if \(\mathcal{M}\) is \((\alpha, \epsilon(\alpha))\)-RDP, then this new randomized algorithm \(\mathcal{M} \circ \mathbf{subsample}\) obeys \(\left(\alpha, \epsilon^{\prime}(\alpha)\right)\)-RDP where,
    \begin{multline} \label{sub_rdp}
        \epsilon^{\prime}(\alpha) \leq \frac{1}{\alpha-1} \log \biggl(1+s^{2 } \binom{\alpha}{2} \min \left\{4\left(e^{\epsilon(2)}-1\right), e^{\epsilon(2)} \min \left\{2,\left(e^{\epsilon(\infty)}-1\right)^{2}\right\}\right\} \\
        +\sum_{j=3}^{\alpha} s^{j} \binom{\alpha}{j} e^{(j-1) \epsilon(j)} \min \left\{2,\left(e^{\epsilon(\infty)}-1\right)^{j}\right\}\biggr).
    \end{multline}
\end{theorem}
For the Gaussian mechanism, we have
\[
    \epsilon(\alpha) = \frac{\alpha}{2\sigma^2},
\]
so \(\epsilon_{\mathcal{M}(\infty)} = \infty\) and the bound simplifies to
\[ \label{sub_rdp_gauss}
    \epsilon^{\prime}(\alpha) \leq \frac{1}{\alpha-1} \log \biggl(1+s^{2 } \binom{\alpha}{2} \min \left\{4\left(e^{1 / \sigma^2}-1\right), 2e^{1 / \sigma^2} \right\} \\
    +\sum_{j=3}^{\alpha} 2s^{j} \binom{\alpha}{j} e^{(j-1)j / (2\sigma^2)} \biggr) =: \epsilon^{\prime}_{\N}(\alpha; \sigma, s),
\]
which we denote as \(\epsilon^{\prime}_{\N}(\alpha; \sigma, s)\).
When \(s\) is small and \(\sigma\) is large, we can discard higher-order terms and write the right-hand side as
\begin{equation} \label{eq:sub_rdp_gauss_approx}
    \epsilon^{\prime}_{\N}(\alpha; \sigma, s) \approx \frac{1}{\alpha-1}  \left(s^2 \frac{\alpha (\alpha-1)}{2} \cdot 4 \frac{1}{\sigma^2}\right) = 2s^2 \frac{\alpha}{\sigma^2},
\end{equation}
where we use the approximation \(e^t \approx 1 + t\) for small \(t\).

\begin{theorem}
    Consider the short step version of the algorithm using subsampling. Given the choice of \(\sigma_f, \sigma_g, \sigma_H\), \(\lambda\) and sampling fraction \(s\). Suppose an actual run of the subsampled algorithm takes \(k_g\) gradient steps and \(k_H\) negative curvature steps. The run is \((\alpha, \epsilon(\alpha))\)-RDP where
    \[
        \epsilon(\alpha) = \frac{\alpha}{2\sigma_f^2} +
        k_g \epsilon^{\prime}_{\N}(\alpha; \sigma_g, s) + k_H \epsilon^{\prime}_{\N}(\alpha; \sigma_{g,H}, s),
    \]
    where \(1/\sigma_{g,H}^2 = 1/\sigma_g^2 + 1/\sigma_H^2\).
    Let
    \begin{equation} \label{eq:sub_rdp_short}
        \bar \epsilon(\alpha) = \frac{\alpha}{2\sigma_f^2} +
        T\epsilon^{\prime}_{\N}(\alpha; \sigma_g, s) + T \epsilon^{\prime}_{\N}(\alpha; \sigma_{g,H}, s).
    \end{equation}
    We always have \(\bar \epsilon(\alpha) \ge \epsilon(\alpha)\), so the algorithm is \((\alpha, \bar \epsilon(\alpha))\)-RDP.\@
\end{theorem}

\begin{proof}
    The proof follows directly from the RDP sampling theorem above and the composition of RDP. 
\end{proof}


Given the complexity in the subsampled privacy guarantee, \(\epsilon^{\prime}_{\N}(\alpha; \sigma, s)\), we do not have an explicit formula to set parameters \(\sigma_f, \sigma_g, \sigma_H\). However, given \((\varepsilon, \delta)\)-DP privacy budget, we can optimize the parameters to meet the privacy guarantee. Recall the conversion from \((\alpha, \epsilon(\alpha))\)-RDP to \((\varepsilon_{DP}, \delta_{DP})\)-DP, given \(\delta\), we solve
\[
    \varepsilon_{DP}(\epsilon(\cdot)) = \min_{\alpha}\left(\epsilon(\alpha) + \frac{\log 1 / \delta_{DP}}{\alpha - 1}\right).
\]
So we can optimize the parameters \(\sigma_f, \sigma_g, \sigma_H\), such that the following objective is minimized
\begin{equation}
    \max \left(\bar \varepsilon_{DP} - \varepsilon_{DP}(\bar \epsilon(\cdot)), 0\right),
\end{equation}
where \(\bar \varepsilon_{DP}\) is the target privacy budget and we replace \(\bar \epsilon(\cdot)\) with their corresponding versions~\eqref{eq:sub_rdp_short}.

\subsection{Sample complexity using \texorpdfstring{\((\varepsilon, \delta)\)-DP}{(ε,δ)-DP}}\label{sec:subsample_analyis_dp}
Under the \((\varepsilon, \delta)\)-DP scheme, subsampling is easier to deal with and we will derive a sample complexity bound.  We first introduce several useful results in \((\varepsilon, \delta)\)-DP.\@

\begin{proposition}[Composition of \((\varepsilon, \delta)\)-DP]
    Suppose that \(\M_1: \mathcal{D} \to \mathcal{R}_{1}\) is \((\varepsilon_1, \delta_1)\)-DP and \(M_2: \mathcal{R}_{1} \times\) \(\mathcal{D} \to \mathcal{R}_{2}\) is \((\varepsilon_2, \delta_2)\)-DP.\@
    Then the mechanism defined as \((X, Y)\), where \(X \sim \M_1(D)\) and \(Y \sim \M_2(X, D)\) is \((\varepsilon_1 + \varepsilon_2, \delta_1 + \delta_2)\)-DP.\@
\end{proposition}

\begin{definition}[Gaussian Mechanism for \((\varepsilon, \delta)\)-DP]
    Given any function \(h: \mathcal{X}^{n} \rightarrow \mathbb{R}^{d}\), the Gaussian Mechanism is defined as:
    \[
        \mathbf{G}_{\sigma}h(D)=h(D) + N(0, \Delta_h^2 \sigma^2 I_d),
    \]
    where \(\Delta_h\) be the \(\ell_{2}\)-sensitivity of the function \(h\) and \(\sigma \geq \frac{\sqrt{2 \ln (1.25 / \delta)} \Delta_h}{\epsilon}\). Then, the Gaussian Mechanism \(\mathbf{G}_{\sigma}h\) satisfies \((\epsilon, \delta)\)-differential privacy.
\end{definition}

\begin{theorem}[Privacy amplification via subsampling \citep{ballePrivacyAmplificationSubsampling2018}]\label{subsampled-dp}.
    Given a dataset of \(n\) points drawn from a domain \(\mathcal{X}\) and a (randomized) mechanism \(\mathcal{M}\) that takes an input from \(\mathcal{X}^{m}\) for \(m \leq n\), let the randomized algorithm \(\mathcal{M} \circ \mathbf{subsample}\) be defined as: (1) subsample: subsample without replacement m datapoints of the dataset (sampling parameter \(s=m/n\)),
    and (2) apply \(\mathcal{M}\): a randomized algorithm taking the subsampled dataset as the input. If \(\mathcal{M}\) is \((\varepsilon, \delta)\)-DP, then \(\mathcal{M} \circ \mathbf{subsample}\) is \((\varepsilon', \delta')\)-DP, where \(\varepsilon' = \log \left(
    1 + s(e^\varepsilon - 1) \le s(e^\varepsilon - 1)
    \right)\) and \(\delta' = s \delta \).
\end{theorem}

\begin{theorem}[Advanced Composition]
    For all \(\varepsilon_0, \delta_0, \delta_0^{\prime} \geq 0\), the class of \((\varepsilon_0, \delta_0)\)-differentially private mechanisms satisfies \(\left(\varepsilon, k \delta_0+\delta_0^{\prime}\right)\)-differential privacy under \(k\)-fold adaptive composition for:
    \[
        \varepsilon=\sqrt{2 k \ln \left(1 / \delta_0^{\prime}\right)} \varepsilon+k \varepsilon\left(e^{\varepsilon}-1\right).
    \]

    As a corollary, for \(0 < \varepsilon < 1\), it suffices to choose
    \(\epsilon_0 = \frac{\varepsilon}{2\sqrt{2k\log (1/ \delta_0')}} \) to ensure the composition is \((\varepsilon, k\delta_0 + \delta_0')\)-DP.\@ In particular, we can in addition choose \(\delta_0' = \delta / 2 \) and \(\delta_0 = \delta / (2k)\) to satisfy \((\varepsilon, \delta)\)-DP.\@
\end{theorem}

We have the following privacy guarantee for the algorithm using sampling without replacement,
\begin{theorem}\label{thm:subsampled-ls-dp}
    Consider the short step version of the algorithm using subsampling. Given privacy parameters \(\varepsilon, \delta, \varepsilon_f, \delta_f \in (0,1)\) such that \(\varepsilon_f < \varepsilon\) and \(\delta_f < \delta\), subsampling parameter \(s\). Let
    \(\varepsilon_0 = (\varepsilon - \varepsilon_f) / (8 s \sqrt{2T\ln(2/(\delta - \delta_f))})\) and \(\delta_0 = (\delta - \delta_f) / (4sT)\), where
    \(T\) is estimated as before in~\eqref{eq:T}.
    \(\sigma_f = \frac{\sqrt{2 \ln (1.25 / \delta_f)}}{\varepsilon_f}\), \(\sigma_g = \sigma_H = \frac{\sqrt{2 \ln (1.25 / \delta_0)} }{\epsilon_0}\).
    The algorithm is \((\varepsilon, \delta)\)-DP.\@
\end{theorem}

\begin{proof}
    By Gaussian mechanism, the step for estimating \(T\) is \((\varepsilon_f, \delta_f)\)-DP.\@ It suffices to show the remaining steps are \((\varepsilon - \varepsilon_f, \delta - \delta_f)\). Using advanced composition, we only need to show that each iteration is \((4s\varepsilon_0, 2s\delta_0)\)-DP.\@

    Consider a single iteration without subsampling. From the usage of Gaussian mechanism and sparse vector technique, we know that computing the perturbed gradient step and
    the perturbed Hessian step are both \((\varepsilon_0, \delta_0)\)-DP.\@
    By composition, we know that the whole iteration is \((2\varepsilon_0, 2\delta_0)\)-DP.\@
    Applying the Privacy Amplification Theorem~\ref{subsampled-dp}, we know that each iteration using subsampling is \((4s\varepsilon_0, 2s\delta_0)\)-DP.\@
\end{proof}

We now discuss the sample complexity.
\begin{theorem}
    For \(c_f \in (0,1)\), setting \(\varepsilon_f = c_f \varepsilon\) and \(\delta_f = c_f \delta\), under the choice of parameters \(\sigma_g, \sigma_H\) in Theorem~\ref{thm:ls_batch}, the asymptotic dependence of \(n_{\min}\) in Theorem~\ref{thm:ls_batch} on \((\epsg, \epsH)\), \((\varepsilon, \delta)\), \(\), is
    \begin{equation}
        n_{\min} = \tilde O \left(
        \frac{\sqrt{d}\ln(1/\delta)}{\varepsilon}
        \max \left(
        \epsg^{-2}, \epsg^{-1} \epsH^{-2}, \epsH^{-4}
        \right)
        \right).
    \end{equation}
    When \((\epsg,\epsH) = (\alpha,\sqrt{M \alpha})\), the dependence simplifies to \(\tilde O (\frac{d \ln(1/\delta)}{\varepsilon \alpha^2})\), matching the result (up to a factor of \(\sqrt{\ln(1/\delta)}\)) in full-batch version of the algorithm by converting \(\rho\)-\(z\)CDP to \((\varepsilon, \delta)\)-DP via \(\sqrt{\rho} = O(\frac{\varepsilon}{\sqrt{\ln(1/\delta)}})\) using Proposition~\ref{prop:zcdp2dp}.
\end{theorem}
\begin{proof}
    As before, we have \(\sqrt{T} = O(\max(\epsg^{-1}, \epsH^{-3/2}))\). After simplification, the order of \(\sigma_g\) and \(\sigma_H\) is
    \[
        \frac{s}{\varepsilon} \sqrt{T} \sqrt{\ln(2sT/\delta) \ln(1/\delta)}
        = \tilde O \left(
        \frac{s}{\varepsilon} \max(\epsg^{-1}, \epsH^{-3/2})
        \right).
    \]
    The asymptotic dependence of \(n_{\min}\) follows by substituting \(\sigma_g\) and \(\sigma_H\) into~\eqref{min_n_subsampled-ls}.
\end{proof}

\section{Computation of the Smallest Eigenvalue Using Lanczos method}\label{sec:lanczos}
In our algorithms, we need to compute the smallest eigenvalue of the perturbed Hessian. This can be done effectively using the randomized Lanczos algorithm. We have the following result from \citep{carmonAcceleratedMethodsNonConvex2017},
\begin{lemma}
    Suppose that the Lanczos method is used to estimate the smallest eigenvalue of \(H\) starting with a random vector uniformly generated on the unit sphere, where \(\|H\| \leq M\). For any \(\delta \in [0,1)\), this approach finds the smallest eigenvalue of \(H\) to an absolute precision of \(\epsilon / 2\), together with a corresponding direction \(v\), in at most
    \[
        \min \left\{d, 1+\left\lceil\frac{1}{2} \ln \left(2.75 d / \delta^2\right) \sqrt{\frac{M}{\epsilon}}\right\rceil\right\} \text { iterations }
    \]
    with probability at least \(1-\delta\).
\end{lemma}
To use Lanczos method in our algorithm, we output an estimate \(\tilde \lambda\) of \(\lambda_{\min}(\tilde H)\) along with the corresponding eigenvector, provided that \(\tilde \lambda \le -\epsH / 2\). If \(\tilde \lambda > -\epsH / 2\), we declare that \(\lambda_{\min}(\tilde H) \ge -\epsH\), with an error probability at most \(\delta_L\). Our analysis and convergence results still hold by replacing \(\epsH\) with \(\epsH / 2\) and adding the success probability of the Lanczos algorithm \(1 - \delta_L\) to the product of the success probability in each iteration.

Note that Lanczos method only requires Hessian-vector products, the cost of which depends on the exact form of the objective. For linear ERM, each term has the form \(l(w^T x_i)\) so the Hessian is a weighted sum of rank-1 terms. A Hessian vector product with the true Hessian can thus be performed in \(O(nd)\) operations, with the cost of multiplying the added noise matrix by the vector costing an additional \(O(d^2)\) operations.
The total cost of randomized Lanczos, dependent on the precision needed, is therefore \(O(nd+d^2)\) times the iteration bound above. 
By contrast, a full Hessian evaluation could cost \(O(nd^2)\) and an eigenvalue factorization would cost \(O(d^3)\).
\footnote{We still evaluate noisy Hessians in our implementation, since there is no optimized support for this in PyTorch.}

\section{Missing Proofs}
\subsection{Proof of Lemma~\ref{lem:short_dec}}\label{proof:short_dec}
\begin{lemma}
    With the short step size choices~\eqref{eq:short_step}, if the noise satisfies the following conditions
    for some positive constants \(c\), \(c_1\), and \(c_2\) such that \(c_1 < \tfrac12\) and \(c_{2}+c < \tfrac13\),
    \begin{subequations}
        \begin{align}
            \left\|\varepsilon_{k}\right\| & \le \min \left(c_{1} \epsg, \frac{c_{2}}{M} \epsH^{2}\right), \\
            \left\|E_{k}\right\|           & \le c \, \epsH,
        \end{align}
    \end{subequations}
    then the amount of decrease in each step is at least
    \begin{equation}
        \textnormal{MIN\_DEC} = \min\left(
        \frac{1-2 c_{1}}{2 G}\epsg^{2}, \, 2\left(\frac{1}{3}-c_{2}-c\right) \frac{\epsH^{3}}{M^{2}}
        \right).
    \end{equation}
    The true gradient and true minimum eigenvalue of the Hessian satisfy the following,
    \begin{equation}
        \left\|g_{k}\right\|
        \le (1+c_1) \left\|\tilde{g}_{k}\right\|, \quad
        \lambda_k > - (1+c) |\tilde \lambda_k|.
    \end{equation}
\end{lemma}
\begin{proof}
    We will use the following two standard bounds, which follow from the smoothness assumptions on \(f\):
    \begin{align}
        f(w+p) & \le f(w)+\nabla f(w)^{\top} p+ \frac{G}{2}\|p\|^{2}, \label{taylor1}                                    \\
        f(w+p) & \le f(w) +\nabla f(w) ^{T}p+\dfrac{1}{2}p^{T}\nabla ^{2}f(w) p+\dfrac{1}{6}M\|p\| ^{3}. \label{taylor2}
    \end{align}
    For simplicity, we drop the iteration number \(k\)  in the analysis below.

    For gradient steps we have \(\left\|\tilde g\right\| > \epsilon_{g}\). We write \(g = \tilde g - \varepsilon\). Using \(\left\|\varepsilon\right\| \le c_{1} \epsg < c_1 \|\tilde g\|\), it follows from~\eqref{taylor1} that
    \[
        \begin{aligned}
            f(w - \gamma_g \tilde g) & \leq f-\gamma_g(\tilde{g}-\varepsilon)^T \tilde{g}+\frac{G}{2}\gamma_g^2 \|\tilde{g}\|^{2} \\
                                     & \leq f-\frac{1}{G}(\tilde{g}-\varepsilon)^T \tilde{g}+\frac{1}{2 G}\|\tilde{g}\|^{2}       \\
                                     & \leq f-\frac{1}{2G}\|\tilde{g}\|^{2}+\frac{1}{G}\left\|\varepsilon\right\|\|\tilde{g}\|    \\
                                     & \leq f-\frac{1}{2G}\|\tilde{g}\|^{2}+\frac{1}{G} c_{1}\|\tilde{g}\|^{2}                    \\
                                     & = f-\frac{1}{2G}\left(1-2 c_{1}\right)\|\tilde{g}\|^{2}                                    \\
                                     & \le f-\frac{1}{2G}\left(1-2 c_{1}\right)\varepsilon_g^{2},
        \end{aligned}
    \]
    while the true gradient satisfies
    \[
        \left\|g\right\| \le \|\tilde{g}\| + \|\varepsilon\|
        \le \left(1+c_1\right) \|\tilde{g}\|.
    \]

    When negative curvature steps are taken, we have \(\tilde \lambda < -\epsH\). By assumption, we have  \(\|\varepsilon\| \le \frac{c_2}{M} \epsH^2 < \frac{c_2}{M} |\tilde \lambda|^2\) and \(\|E\| \le  c \, \epsH < c |\tilde \lambda|\).
    Recall the definition~\eqref{eq:min_eig} of \(\tilde p\) and we write \(g = \tilde g - \varepsilon\), \(\tilde H = H - E\). From~\eqref{taylor2}, we have
    \[
        \begin{aligned}
            f(w+ \gamma_H \tilde p) & \leq f+\gamma_H g^{T} \tilde p+\frac{1}{2} \gamma_H^{2}\tilde p^{T} H \tilde p+\frac{1}{6} M \gamma_H^{3}\left\|\tilde p\right\|^{3}                                                                                                                                                    \\
                                    & = f+ \gamma_H \tilde g^{T} \tilde p+\frac{1}{2} \gamma_H^{2}\tilde p^{T} \tilde H \tilde p+\frac{1}{6} M \gamma_H^{3}\left\|\tilde p\right\|^{3} - \gamma_H \varepsilon^{T} \tilde p  - \frac{1}{2} \gamma_H^{2}\tilde p^{T} E \tilde p                                                 \\
                                    & \le f + \frac{1}{2} \left(\frac{2|\tilde\lambda|}{M}\right)^{2} (- |\tilde\lambda|) + \frac{1}{6} M \left(\frac{2|\tilde\lambda|}{M}\right)^{3} - \frac{2|\tilde\lambda|}{M} \varepsilon^{T} \tilde p  - \frac{1}{2} \left(\frac{2|\tilde\lambda|}{M}\right)^{2}\tilde p^{T} E \tilde p \\
                                    & \le f-\frac{2}{3} \frac{|\tilde \lambda|^{3}}{M^{2}} + \frac{2|\tilde\lambda|}{M} \left\|\varepsilon\right\|  + \frac{2|\tilde\lambda|^2}{M^2} \|E\|                                                                                                                                    \\
                                    & \le f-\left(\frac{2}{3} - 2c_2 - 2c\right) \frac{|\tilde \lambda|^{3}}{M^{2}}                                                                                                                                                                                                           \\
                                    & \le f-\left(\frac{2}{3} - 2c_2 - 2c\right) \frac{\epsH^{3}}{M^{2}},
        \end{aligned}
    \]
    provided that \(c_2 + c < 1/3\).
    Let \(\lambda\) denote the minimum eigenvalue of \(H\). It follows from Weyl's Inequality that
    \[
        |\tilde \lambda - \lambda| \le \|E\| \le c |\tilde \lambda|,
    \]
    and thus,
    \[
        \lambda > \tilde \lambda - c |\tilde \lambda| \ge - \left(1+c\right) |\tilde \lambda|.
    \]
\end{proof}

\subsection{Proof of Corollary~\ref{cor:output}}
\begin{corollary}
    Assuming the noise satisfies~\eqref{eq:bounded_noise} at each iteration, the short step version~(using~\eqref{eq:short_step},~\eqref{eq:short_dec}) of the algorithm will output a \(((1 + c_{1}) \epsg, (1+c) \epsH)\)-2S.\@
\end{corollary}
\begin{proof}
    From the minimum decrease~\eqref{eq:short_dec} we just derived, it follows that the algorithm will terminate in \(T^*\) iterations, where
    \[
        T^* = \frac{f(w_0)-f^*}{\text{MIN\_DEC}}.
    \]
    Our choice of \(T\) in~\eqref{eq:T} is an upper bound of \(T^*\) and thus the algorithm will halt within \(T\) iterations.
    In the iteration \(k\) when the algorithm halts, we have \(\left\|\tilde g_k\right\| \le \epsilon_{g}\) and \(\tilde \lambda_{k} \ge -\epsH\).
    It follows from~\eqref{eq:close} that the output is a \(((1+c_{1}) \epsg, (1+c) \epsH)\)-2S.\@
\end{proof}

\subsection{Proof of Lemma~\ref{lem:sd}}
\begin{lemma}
    Assume the same bounded noise conditions~\eqref{eq:bounded_noise} as before. With the choice of sufficient decrease coefficients \(c_g \in (0, 1 - c_1), c_{H} \in (0, 1-c-\sqrt{\frac{8}{3} c_{2}})\), let \(\bar \gamma_g = 2\left(1-c_{1}-c_{g}\right)/G \) and
    \(\bar \gamma_H = t_2|\tilde{\lambda}| / {M}\) as defined in Algorithm~\ref{alg:dpopt_ls}, the sufficient decrease conditions~\eqref{SD1} and~\eqref{SD2} are satisfied when \(\gamma \le \bar \gamma_g\) and \(\gamma \in [(t_1 / t_2) \bar \gamma_H, \bar \gamma_H]\), respectively,
    where \(0 < t_1 <  t_2\) are solutions to the following quadratic equation (given our choice of \(c, c_2, c_H\), real solutions exist),
    \[
        r(t) := -\frac{1}{6} t^{2}+\frac{1}{2}\left(1-c-c_{H}\right) t-c_{2} = 0,
    \]
    Explicitly, we have
    \begin{equation}
        t_{1}, t_{2}=\frac{3}{2}\left(1-c-c_{H}\right) \pm 3 \sqrt{\frac{1}{4}(1-c - c_H)^{2}-\frac{2}{3} c_{2}}.
    \end{equation}
    In particular, we have \(q_g(\bar \gamma_g) \ge 0\) and \(q_H(\bar \gamma_H) \ge 0\).
\end{lemma}
\begin{proof}

    The analysis is similar to~\ref{proof:short_dec}.
    Again for simplicity, we drop iteration indices \(k\).

    For gradient steps we have \(\left\|\tilde g\right\| > \epsilon_{g}\). We write \(g = \tilde g - \varepsilon\). Using \(\left\|\varepsilon\right\| \le c_{1} \epsg < c_1 \|\tilde g\|\), it follows from~\eqref{taylor1} that
    \[
        \begin{aligned}
            f(w-\gamma \tilde{g}) & \le f(w)-\gamma(\tilde{g}-\varepsilon)^{\top} \tilde{g}+\frac{G}{2} \gamma^{2}\|\tilde{g}\|^{2}         \\
                                  & \le f(w)-\left(\gamma-\frac{G}{2} \gamma^{2}\right)\|\tilde{g}\|^{2}+\gamma\|\tilde{g}\|\|\varepsilon\| \\
                                  & \le f(w)-\gamma\left(1-\frac{G}{2} \gamma-c_{1}\right) \|\tilde{g}\|^{2},
        \end{aligned}
    \]
    It follows by definition of \(\bar \gamma_g\) that~\eqref{SD1} holds when \(\gamma \le \bar \gamma_g\) and

    When negative curvature steps are taken, we have \(\tilde \lambda < -\epsH\). By assumption, we have  \(\|\varepsilon\| \le \frac{c_2}{M} \epsH^2 < \frac{c_2}{M} |\tilde \lambda|^2\) and \(\|E\| \le  c \, \epsH < c |\tilde \lambda|\).
    Recall the definition~\eqref{eq:min_eig} of \(\tilde p\) and we write \(g = \tilde g - \varepsilon\), \(\tilde H = H - E\). From~\eqref{taylor2}, we have for \(\gamma > 0\) that
    \[
        \begin{aligned}
            f(w+\gamma \tilde{p}) & \le f(w)+\gamma \tilde{g}^{T} \tilde{p}+\frac{1}{2} \gamma^{2} \tilde {p}^{T} H \tilde {p}+\frac{1}{6} M \gamma^{3}\| \tilde {p} \|^{3}    -\gamma {\varepsilon}^{T} \tilde{p}-\frac{1}{2} \gamma^{2} \tilde{p}^{T} E \tilde{p} \\
                                  & \le f(w)-\frac{1}{2} \gamma^{2}| \tilde {\lambda} |+\frac{1}{6} M \gamma^{3}+  \gamma\|\varepsilon\|+\frac{1}{2} \gamma^{2}\|E\|                                                                                                \\
                                  & \le f(w)-\underbrace{\left(\frac{1}{2} \gamma^{2}(1-c)|\tilde{\lambda}|-\gamma \frac{c_{2}}{M}|\tilde{\lambda}|^{2}-\frac{1}{6} M \gamma^{3}\right)}_{g(\gamma)}.
        \end{aligned}
    \]
    By reparameterizing \(\gamma=\frac{t|\tilde{\lambda}|}{M}\), we obtain
    \[
        g(\gamma)-\frac{1}{2} c_H \gamma^{2}|\tilde{\lambda}| = \left(-\frac{1}{6} t^{3} + \frac{1}{2}\left(1-c-c_{H}\right) t^{2}- c_{2}t\right)\frac{|\tilde{\lambda}|^{3}}{M^{2}} = t \cdot r(t)\frac{|\tilde{\lambda}|^{3}}{M^{2}}.
    \]
    Note that~\eqref{SD2} holds when \(g(\gamma) \ge \frac{1}{2} c_H \gamma^{2}|\tilde{\lambda}|\). The result follows from the fact that \(r(t) \ge 0\) for \(t \in [t_{1}, t_{2}]\).

\end{proof}

\subsection{Proof of Lemma~\ref{lem:ls_decrease}}
\begin{lemma}
    Using DP line search Algorithm~\ref{alg:dpopt_ls}, assume the same bounded noise conditions~\eqref{eq:bounded_noise} as before. With the choice of sufficient decrease coefficients \(c_g \in (0, 1 - c_1), c_{H} \in (0, 1-c-\sqrt{\frac{8}{3} c_{2}})\), define \(\bar \gamma_g\) and \(\bar \gamma_H\) as before. Choose initial step size multipliers \( b_g, b_H > 1 \)
    and decrease parameters \(\beta_g \in (0,1), \beta_H \in (t_1/t_2, 1) \).
    Let \(i_{\max} = \lfloor\log_{\beta}\max (b_g, b_H)\rfloor + 1\). If \(n\) satisfies the following:
    \begin{equation}
        n \ge 16\lambda \left(\log i_{\max}  + \log (T / \xi)\right) \max \left(
        2 b_g \frac{B_g}{c_g\epsg},
        4 b_H \frac{B_g M}{t_2 c_{H} \epsH^2}
        \right),
    \end{equation}
    with probability at least \(1- \xi / T\),
    the amount of decrease in a single step is at least
    \begin{equation}
        \textnormal{MIN\_DEC} =
        \min \left(
        \frac1G (1-c_1-c_g) c_g \epsg^2,
        \frac14 c_{H} t_2^2 \frac{\epsH^{3}}{M^{2}}
        \right).
    \end{equation}
\end{lemma}

\begin{proof}

    For the loss function \(l\) in the ERM setting, we have by the mean value theorem that there exists \(t \in (0,1)\) for which
    \[
        |l(w, x) - l(w-\gamma \tilde{g}, x)| = |\gamma \tilde{g}^T \nabla_w l(w-t\gamma \tilde{g}, x)| \le \gamma B_g \|\tilde{g}\|.
    \]
    We start by analyzing the sensitivity of \(q_g\) and \(q_H\). From the definition of \(q_g(\gamma,w)\) in~\eqref{eq:defqg} and the bound above that
    \begin{align*}
        |q_g(\gamma, w, D) - q_g(\gamma, w, D')| & = \frac{1}{n}|
        l(w, x_k) - l(w-\gamma \tilde{g}, x_k) -
        l(w, x_k') - l(w-\gamma \tilde{g}, x_k')|                                           \\
                                                 & \le \frac{1}{n}|
        l(w, x_k) - l(w-\gamma \tilde{g}, x_k)| +
        |l(w, x_k') + l(w-\gamma \tilde{g}, x_k')|                                          \\
                                                 & \le \frac{2}{n}\gamma B_g \|\tilde{g}\|.
    \end{align*}
    We have \(\gamma \le \gamma_{g}^{\text{init}} = b_g \bar \gamma_g\). Thus for the sensitivity of \(q_g\), we have
    \begin{equation} \label{eq:Delqg}
        \Delta_{q_g} = \frac{2}{n}\gamma_{g}^{\text{init}} B_g \|\tilde{g}\| = \frac{2}{n} b_g \bar \gamma_{g} B_g \|\tilde{g}\|.
    \end{equation}
    For \(q_H\), using the definition~\eqref{eq:defqH} and by a similar argument to the one above, we obtain the following upper bound on the sensitivity:
    \begin{equation} \label{eq:DelqH}
        \Delta_{q_H} = \frac{2}{n}\gamma_{H}^{\text{init}} B_g \|p\| = \frac{2}{n} b_H \bar \gamma_{H} B_g.
    \end{equation}

    By the property of SVT (cf. \citet[Theorem~3.24]{dworkAlgorithmicFoundationsDifferential2014}), we know that with probability at least \(1 - \xi / T\), \textsc{LineSearch} will output a \(\gamma\) such that
    \[
        q(\gamma) \ge - t \Delta_q, \quad \text{where} \; t = 8 \lambda \left(\log i_{\max} + \log (T / \xi)\right).
    \]
    Here, \(i_{\max}\) is the maximum number of iterations for \textsc{DP-LineSearch}, which should be \(\lfloor\log_{\beta}b_g\rfloor + 1\) and \(\lfloor\log_{\beta}b_H\rfloor + 1\) for gradient and negative curvature steps, respectively. We take \(i_{\max}\) as the maximum of the two so the bound will hold for both cases.
    (If we end up with a fallback value for the step size, the final output \(\bar \gamma\) will automatically satisfy the condition above, since by definition we have \(q(\bar \gamma) \ge 0\).)
    Rewrite our assumptions of \(n\) using \(t\),
    \[
        2t b_g \frac{B_g}{n} \le \frac12 c_g \epsg, \quad 2t b_H \frac{B_g}{n} \le \frac14 c_H \epsH \cdot \frac{t_2 \epsH}{M}.
    \]
    If follows by combinint with~\eqref{eq:Delqg} and~\eqref{eq:DelqH} that
    \[
        \begin{aligned}
            t \Delta_{q_g} & = 2t b_g \frac{B_g}{n}  \bar \gamma_{g} \|\tilde{g}\|
            \le \frac12 c_g \epsg \bar \gamma_{g} \|\tilde{g}\|  \le \frac12 c_g \bar \gamma_{g} \|\tilde{g}\|^2,                                                      \\
            t \Delta_{q_H} & = 2t b_H \frac{B_g}{n} \bar \gamma_H \le \frac14 c_H \epsH \bar \gamma_H \cdot \frac{t_2 \epsH}{M} \le \frac14 c_H \epsH \bar \gamma_H^2,
        \end{aligned}
    \]
    where we use the preconditions for gradient steps and negative curvature steps, \(\left\|\tilde g\right\| > \epsilon_{g}\) and \(\tilde \lambda < -\epsH\), respectively, and the fact that \(\bar \gamma_H = \frac{t_2 |\tilde \lambda|}{M} \ge \frac{t_2 \epsH}{M}\) by definition. For the output \(\gamma\), we always have \(\gamma \ge \bar \gamma\), since we decrease \(\gamma\) at most \(\lfloor\log_{\beta}\max b \rfloor + 1\) times. It follows that, the amount of decrease is at least
    \[
        \begin{aligned}
            c_g \gamma_g \|\tilde{g}\|^2 - t \Delta_{q_g}                 & \ge \frac12 c_g \bar \gamma_g \|\tilde{g}\|^2 \ge
            \frac12 c_g \bar \gamma_g \epsg^2
            = \frac1G (1-c_1-c_g) c_g \epsg^2,
            \\
            \frac12 c_H \gamma_H^{2} \|\tilde{\lambda}\| - t \Delta_{q_H} & \ge
            \frac14 c_{H} \epsH \bar \gamma_H^2 \ge
            \frac14 c_{H} t_2^2 \frac{\epsH^{3}}{M^{2}}                   ,
        \end{aligned}
    \]
    for gradient steps and negative curvature steps, respectively. The result follows by taking the minimum of two quantities on the right hand sides.

\end{proof}

\section{Experimental Settings and Additional Experiments}\label{sec:exp_details}
The algorithms are implemented using PyTorch. For privacy accounting of RDP, which is used in the mini-batched algorithm, we use the \texttt{autodp} package\footnote{Open source repo: \url{https://github.com/yuxiangw/autodp}}.
All our experiments were carried out on a cluster with 36-core Intel Xeon Gold 6254 3.1GHz CPUs, utilizing
8 CPU cores for each run.
\subsection{Datasets}
The Covertype dataset\footnote{Data source: UCI Machine Learning Repository \url{https://archive.ics.uci.edu/ml/datasets/covertype}} contains \(n=581012\) data points.
Each data point has the form \((x, y)\), where \(x\) is a 54-dimensional feature vector (first 10 are dimensions numerical, column 11 – 14 is the WildernessArea one-hot vector, and last 40 columns are the SoilType one-hot vector), and \(y\) being the label, is one of \(\{1, 2, \ldots, 7\}\).

For preprocessing, we normalize the first 10 numerical columns, and keep only those samples for which  \(y=1,2\).
The number of samples remaining in this restricted set is \(n=495141\).
We recode \(y=2\) to \(y=-1\) so that \(y \in \{-1,1\}\).

The IJCNN dataset\footnote{Data source: LIBSVM data repository \url{https://www.openml.org/search?type=data&sort=runs&id=1575&status=active}} contatins \(n=4999\) data points. Each point consists of \((x, y)\), where \(x\) is a 22-dimensional feature vector (first 10 are one hot and column 11 – 22 are numerical, and \(y\) being the label is binary). For preprocessing, we normalize the data.

Below we repeat the same experiment using the IJCNN dataset.
\subsection{IJCNN experiment using loss in Section~\ref{sec:exp}}

\begin{table}[!htb]
    \centering
    \caption{IJCNN:\@ finding a loose solution, \((\epsg, \epsH) = (0.040, 0.200)\)}
\begin{tabular}{ccccccc}
    \toprule
    \multirow{2}{*}{method} & \multicolumn{2}{c}{\(\varepsilon = 0.2\)}                                     & \multicolumn{2}{c}{\(\varepsilon = 0.6\)}  & \multicolumn{2}{c}{\(\varepsilon = 1.0\)}            \\
    \cmidrule(lr){2-3}
    \cmidrule(lr){4-5}
    \cmidrule(lr){6-7}
    & final loss & runtime & loss & runtime & loss & runtime \\
    \midrule
    TR      &    \(0.621 \pm 0.009\) & \( \times \) &   \(0.71 \pm 0.025\) & \(1.3 \pm 1.5\) &  \(0.718 \pm 0.019\) & \(0.8 \pm 1.1\) \\
    TR-B    &    \(0.622 \pm 0.009\) & \( \times \) &  \(0.718 \pm 0.046\) & \(0.4 \pm 0.4\) &   \(0.72 \pm 0.043\) & \(1.3 \pm 1.6\) \\
\midrule
    OPT     &    \(0.603 \pm 0.012\) & \( \times \) &     \(0.644 \pm 0.014\) & \( \times \) &  \(0.702 \pm 0.015\) & \(0.1 \pm 0.0\) \\
    OPT-B   &  \(0.71 \pm 0.022\) & \(2.7 \pm 2.4\) &   \(0.71 \pm 0.022\) & \(2.6 \pm 2.5\) &   \(0.71 \pm 0.022\) & \(2.8 \pm 2.3\) \\
    OPT-LS  &     \(0.671 \pm 0.03\) & \( \times \) &      \(0.593 \pm 0.02\) & \( \times \) &     \(0.658 \pm 0.019\) & \( \times \) \\
\midrule
    2OPT    &    \(0.631 \pm 0.052\) & \( \times \) &     \(0.679 \pm 0.048\) & \( \times \) &     \(0.676 \pm 0.056\) & \( \times \) \\
    2OPT-B  &  \(0.71 \pm 0.022\) & \(1.0 \pm 0.3\) &   \(0.71 \pm 0.022\) & \(1.0 \pm 0.3\) &   \(0.71 \pm 0.022\) & \(1.1 \pm 0.3\) \\
    2OPT-LS &  NA &            NA  &  \(0.696 \pm 0.013\) & \(0.045 \pm 0.005\) &  \(0.693 \pm 0.013\) & \(0.047 \pm 0.005\) \\
    \bottomrule
\end{tabular}
\end{table}

\begin{table}[!htb]
    \centering
    \caption{IJCNN:\@ finding a tight solution: \((\epsg, \epsH) = (0.020, 0.141)\)}
    \begin{tabular}{ccccccc}
        \toprule
        \multirow{2}{*}{method} & \multicolumn{2}{c}{\(\varepsilon = 0.2\)}                                     & \multicolumn{2}{c}{\(\varepsilon = 0.6\)}  & \multicolumn{2}{c}{\(\varepsilon = 1.0\)}            \\
        \cmidrule(lr){2-3}
        \cmidrule(lr){4-5}
        \cmidrule(lr){6-7}
        & final loss & runtime & loss & runtime & loss & runtime \\ \midrule
    TR      &     \(0.625 \pm 0.008\) & \( \times \) &      \(0.59 \pm 0.009\) & \( \times \) &      \(0.57 \pm 0.009\) & \( \times \) \\
TR-B    &     \(0.625 \pm 0.009\) & \( \times \) &     \(0.591 \pm 0.009\) & \( \times \) &     \(0.574 \pm 0.009\) & \( \times \) \\
\midrule
OPT     &     \(0.643 \pm 0.011\) & \( \times \) &     \(0.583 \pm 0.019\) & \( \times \) &     \(0.515 \pm 0.039\) & \( \times \) \\
OPT-B   &  \(0.667 \pm 0.013\) & \(1.1 \pm 0.3\) &  \(0.667 \pm 0.013\) & \(0.9 \pm 0.1\) &  \(0.667 \pm 0.013\) & \(1.0 \pm 0.1\) \\
OPT-LS  &     \(0.835 \pm 0.069\) & \( \times \) &     \(0.635 \pm 0.022\) & \( \times \) &     \(0.598 \pm 0.024\) & \( \times \) \\
\midrule
2OPT    &     \(0.537 \pm 0.027\) & \( \times \) &     \(0.522 \pm 0.074\) & \( \times \) &     \(0.552 \pm 0.091\) & \( \times \) \\
2OPT-B  &  \(0.666 \pm 0.014\) & \(1.1 \pm 0.4\) &  \(0.666 \pm 0.014\) & \(1.3 \pm 0.6\) &  \(0.666 \pm 0.014\) & \(1.1 \pm 0.3\) \\
2OPT-LS &  NA &            NA &      \(0.623 \pm 0.06\) & \( \times \) &  \(0.649 \pm 0.004\) & \(0.1 \pm 0.1\) \\
        \bottomrule
    \end{tabular}
\end{table}

We remark that for 2OPT-LS under \(\varepsilon = 0.2\), the result is unavailable (reported as NA) because due to numerical issues, the package \texttt{autodp} we use cannot handle subsampling with a very low privacy budget.

\subsection{Additional Experiments}\label{sec:add_exp}

Additionally, we consider the logistic loss
\[
\nabla l(w) = \frac{1}{n} \sum_{i=1}^{n} \frac{1}{1+\exp \left(-y_{i}\left\langle x_{i}, w\right\rangle\right)} + \frac{\lambda}{2} \|w\|^2,
\]
and repeat our experiments on the aforementioned datasets with \(\lambda = 10^{-3}\).
We can verify that the two chosen losses have Lipschitz gradients and Hessians as long as the feature vector \(x_i\)'s are bounded.

In this set of experiments, we find solutions \((\epsg, \epsH) = (0.040, 0.200)\) and \((\epsg, \epsH) = (0.020, 0.141)\).
We also show the aggregated results for the number of noisy Hessian evaluations. We note that the number of noisy Hessian evaluations required in our algorithm is very low, whereas DP-TR needs to evaluate the noisy Hessian every iteration.

\subsubsection{Covertype experiment using logistic loss}
\begin{table}[!htb]
    \centering
    \caption{Covertype (logistic loss): finding a loose solution, \((\epsg, \epsH) = (0.040, 0.200)\)}
\begin{tabular}{ccccccc}
    \toprule
    \multirow{2}{*}{method} & \multicolumn{2}{c}{\(\varepsilon = 0.2\)}                                     & \multicolumn{2}{c}{\(\varepsilon = 0.6\)}  & \multicolumn{2}{c}{\(\varepsilon = 1.0\)}            \\
    \cmidrule(lr){2-3}
    \cmidrule(lr){4-5}
    \cmidrule(lr){6-7}
    & final loss & runtime & loss & runtime & loss & runtime \\
    \midrule
TR      &      \(0.425 \pm 0.009\) & \( \times \) &      \(0.388 \pm 0.001\) & \( \times \) &      \(0.381 \pm 0.001\) & \( \times \) \\
TR-B    &      \(0.425 \pm 0.009\) & \( \times \) &      \(0.388 \pm 0.002\) & \( \times \) &      \(0.382 \pm 0.001\) & \( \times \) \\
\midrule
OPT     &      \(0.442 \pm 0.006\) & \( \times \) &   \(0.539 \pm 0.022\) & \(0.3 \pm 0.1\) &   \(0.539 \pm 0.022\) & \(0.4 \pm 0.1\) \\
OPT-B   &  \(0.539 \pm 0.022\) & \(11.1 \pm 1.7\) &  \(0.539 \pm 0.022\) & \(10.6 \pm 0.4\) &  \(0.539 \pm 0.022\) & \(10.6 \pm 0.4\) \\
OPT-LS  &      \(0.385 \pm 0.002\) & \( \times \) &      \(0.455 \pm 0.014\) & \( \times \) &   \(0.539 \pm 0.022\) & \(0.4 \pm 0.1\) \\
\midrule
2OPT    &   \(0.539 \pm 0.022\) & \(0.3 \pm 0.1\) &   \(0.539 \pm 0.022\) & \(0.3 \pm 0.1\) &   \(0.539 \pm 0.022\) & \(0.4 \pm 0.1\) \\
2OPT-B  &   \(0.539 \pm 0.022\) & \(1.0 \pm 0.0\) &   \(0.539 \pm 0.022\) & \(1.0 \pm 0.1\) &   \(0.539 \pm 0.022\) & \(1.0 \pm 0.1\) \\
2OPT-LS &   \(0.539 \pm 0.022\) & \(0.4 \pm 0.1\) &   \(0.539 \pm 0.022\) & \(0.3 \pm 0.1\) &   \(0.539 \pm 0.022\) & \(0.4 \pm 0.1\) \\
    \bottomrule
\end{tabular}
\end{table}

\begin{table}[!htb]
    \centering
    \caption{Covertype Hess evals (logistic loss): finding a loose solution, \((\epsg, \epsH) = (0.040, 0.200)\)}
\begin{tabular}{ccccccc}
    \toprule
    \multirow{2}{*}{method} & \multicolumn{2}{c}{\(\varepsilon = 0.2\)}                                     & \multicolumn{2}{c}{\(\varepsilon = 0.6\)}  & \multicolumn{2}{c}{\(\varepsilon = 1.0\)}            \\
    \cmidrule(lr){2-3}
    \cmidrule(lr){4-5}
    \cmidrule(lr){6-7}
    & Hess evals & runtime & Hess evals & runtime & Hess evals & runtime \\
    \midrule
TR      &    \(375.0 \pm 0.0\) & \( \times \) &      \(375.0 \pm 0.0\) & \( \times \) &    \(375.0 \pm 0.0\) & \( \times \) \\
TR-B    &    \(375.0 \pm 0.0\) & \( \times \) &      \(375.0 \pm 0.0\) & \( \times \) &    \(375.0 \pm 0.0\) & \( \times \) \\
\midrule
OPT     &   \(190.6 \pm 32.3\) & \( \times \) &     \(1.0 \pm 0.0\) & \(0.3 \pm 0.1\) &   \(1.0 \pm 0.0\) & \(0.4 \pm 0.1\) \\
OPT-B   &  \(1.0 \pm 0.0\) & \(11.1 \pm 1.7\) &    \(1.0 \pm 0.0\) & \(10.6 \pm 0.4\) &  \(1.0 \pm 0.0\) & \(10.6 \pm 0.4\) \\
OPT-LS  &      \(0.0 \pm 0.0\) & \( \times \) &  \(282.4 \pm 174.252\) & \( \times \) &   \(1.0 \pm 0.0\) & \(0.4 \pm 0.1\) \\
\midrule
2OPT    &   \(1.0 \pm 0.0\) & \(0.3 \pm 0.1\) &     \(1.0 \pm 0.0\) & \(0.3 \pm 0.1\) &   \(1.0 \pm 0.0\) & \(0.4 \pm 0.1\) \\
2OPT-B  &   \(1.0 \pm 0.0\) & \(1.0 \pm 0.0\) &     \(1.0 \pm 0.0\) & \(1.0 \pm 0.1\) &   \(1.0 \pm 0.0\) & \(1.0 \pm 0.1\) \\
2OPT-LS &   \(1.0 \pm 0.0\) & \(0.4 \pm 0.1\) &     \(1.0 \pm 0.0\) & \(0.3 \pm 0.1\) &   \(1.0 \pm 0.0\) & \(0.4 \pm 0.1\) \\
    \bottomrule
\end{tabular}
\end{table}

\begin{table}[!htb]
    \centering
    \caption{Covertype (logistic loss): finding a tight solution: \((\epsg, \epsH) = (0.020, 0.141)\)}
    \begin{tabular}{ccccccc}
        \toprule
        \multirow{2}{*}{method} & \multicolumn{2}{c}{\(\varepsilon = 0.2\)}                                     & \multicolumn{2}{c}{\(\varepsilon = 0.6\)}  & \multicolumn{2}{c}{\(\varepsilon = 1.0\)}            \\
        \cmidrule(lr){2-3}
        \cmidrule(lr){4-5}
        \cmidrule(lr){6-7}
        & final loss & runtime & loss & runtime & loss & runtime \\ \midrule
TR      &     \(0.408 \pm 0.004\) & \( \times \) &       \(0.381 \pm 0.0\) & \( \times \) &       \(0.378 \pm 0.0\) & \( \times \) \\
TR-B    &     \(0.408 \pm 0.004\) & \( \times \) &       \(0.381 \pm 0.0\) & \( \times \) &       \(0.378 \pm 0.0\) & \( \times \) \\
\midrule
OPT     &     \(0.379 \pm 0.001\) & \( \times \) &      \(0.38 \pm 0.001\) & \( \times \) &      \(0.39 \pm 0.002\) & \( \times \) \\
OPT-B   &  \(0.454 \pm 0.004\) & \(1.3 \pm 0.2\) &  \(0.454 \pm 0.004\) & \(1.4 \pm 0.4\) &  \(0.454 \pm 0.004\) & \(1.5 \pm 0.6\) \\
OPT-LS  &      \(0.41 \pm 0.005\) & \( \times \) &     \(0.381 \pm 0.001\) & \( \times \) &       \(0.378 \pm 0.0\) & \( \times \) \\
\midrule
2OPT    &     \(0.386 \pm 0.006\) & \( \times \) &       \(0.378 \pm 0.0\) & \( \times \) &       \(0.377 \pm 0.0\) & \( \times \) \\
2OPT-B  &  \(0.454 \pm 0.004\) & \(2.0 \pm 0.2\) &  \(0.454 \pm 0.004\) & \(2.0 \pm 0.3\) &  \(0.454 \pm 0.004\) & \(1.8 \pm 0.2\) \\
2OPT-LS &     \(0.441 \pm 0.007\) & \( \times \) &  \(0.447 \pm 0.009\) & \(0.6 \pm 0.2\) &  \(0.447 \pm 0.008\) & \(0.7 \pm 0.2\) \\
    \bottomrule
    \end{tabular}
\end{table}

\begin{table}[!htb]
    \centering
    \caption{Covertype Hess evals (logistic loss): finding a tight solution: \((\epsg, \epsH) = (0.020, 0.141)\)}
    \begin{tabular}{ccccccc}
        \toprule
        \multirow{2}{*}{method} & \multicolumn{2}{c}{\(\varepsilon = 0.2\)}                                     & \multicolumn{2}{c}{\(\varepsilon = 0.6\)}  & \multicolumn{2}{c}{\(\varepsilon = 1.0\)}            \\
        \cmidrule(lr){2-3}
        \cmidrule(lr){4-5}
        \cmidrule(lr){6-7}
        & Hess evals & runtime & Hess evals & runtime & Hess evals & runtime \\ \midrule
\midrule
TR      &  \(1061.0 \pm 0.0\) & \( \times \) &   \(1061.0 \pm 0.0\) & \( \times \) &     \(1061.0 \pm 0.0\) & \( \times \) \\
TR-B    &  \(1061.0 \pm 0.0\) & \( \times \) &   \(1061.0 \pm 0.0\) & \( \times \) &     \(1061.0 \pm 0.0\) & \( \times \) \\
\midrule
OPT     &     \(0.0 \pm 0.0\) & \( \times \) &  \(58.6 \pm 24.765\) & \( \times \) &  \(372.6 \pm 141.077\) & \( \times \) \\
OPT-B   &  \(1.0 \pm 0.0\) & \(1.3 \pm 0.2\) &   \(1.0 \pm 0.0\) & \(1.4 \pm 0.4\) &     \(1.0 \pm 0.0\) & \(1.5 \pm 0.6\) \\
OPT-LS  &     \(0.0 \pm 0.0\) & \( \times \) &      \(0.0 \pm 0.0\) & \( \times \) &        \(1.0 \pm 1.0\) & \( \times \) \\
\midrule
2OPT    &     \(0.0 \pm 0.0\) & \( \times \) &      \(0.0 \pm 0.0\) & \( \times \) &      \(0.2 \pm 0.447\) & \( \times \) \\
2OPT-B  &  \(1.0 \pm 0.0\) & \(2.0 \pm 0.2\) &   \(1.0 \pm 0.0\) & \(2.0 \pm 0.3\) &     \(1.0 \pm 0.0\) & \(1.8 \pm 0.2\) \\
2OPT-LS &  \(33.0 \pm 6.745\) & \( \times \) &   \(1.0 \pm 0.0\) & \(0.6 \pm 0.2\) &     \(1.0 \pm 0.0\) & \(0.7 \pm 0.2\) \\
    \bottomrule
    \end{tabular}
\end{table}

\newpage

\subsubsection{IJCNN experiment using logistic loss}
\begin{table}[!htb]
    \centering
    \caption{IJCNN (logistic loss): finding a loose solution, \((\epsg, \epsH) = (0.040, 0.200)\)}
\begin{tabular}{ccccccc}
    \toprule
    \multirow{2}{*}{method} & \multicolumn{2}{c}{\(\varepsilon = 0.2\)}                                     & \multicolumn{2}{c}{\(\varepsilon = 0.6\)}  & \multicolumn{2}{c}{\(\varepsilon = 1.0\)}            \\
    \cmidrule(lr){2-3}
    \cmidrule(lr){4-5}
    \cmidrule(lr){6-7}
    & final loss & runtime & loss & runtime & loss & runtime \\
    \midrule
TR      &      \(0.477 \pm 0.01\) & \( \times \) &      \(0.454 \pm 0.006\) & \( \times \) &     \(0.446 \pm 0.004\) & \( \times \) \\
TR-B    &      \(0.477 \pm 0.01\) & \( \times \) &      \(0.454 \pm 0.006\) & \( \times \) &     \(0.446 \pm 0.004\) & \( \times \) \\
\midrule
OPT     &       \(0.46 \pm 0.01\) & \( \times \) &      \(0.439 \pm 0.002\) & \( \times \) &     \(0.463 \pm 0.005\) & \( \times \) \\
OPT-B   &  \(0.501 \pm 0.008\) & \(8.6 \pm 0.2\) &  \(0.501 \pm 0.008\) & \(11.0 \pm 4.9\) &  \(0.501 \pm 0.008\) & \(9.0 \pm 1.0\) \\
OPT-LS  &     \(0.607 \pm 0.076\) & \( \times \) &      \(0.473 \pm 0.012\) & \( \times \) &     \(0.454 \pm 0.004\) & \( \times \) \\
\midrule
2OPT    &     \(0.493 \pm 0.022\) & \( \times \) &   \(0.501 \pm 0.008\) & \(0.0\footnote{Due to round off, the same below} \pm 0.0\) &  \(0.501 \pm 0.008\) & \(0.0 \pm 0.0\) \\
2OPT-B  &  \(0.501 \pm 0.008\) & \(1.0 \pm 0.0\) &   \(0.501 \pm 0.008\) & \(1.4 \pm 0.9\) &  \(0.501 \pm 0.008\) & \(1.1 \pm 0.3\) \\
2OPT-LS &     \(1.002 \pm 0.321\) & \( \times \) &   \(0.501 \pm 0.008\) & \(0.0 \pm 0.0\) &  \(0.501 \pm 0.008\) & \(0.0 \pm 0.0\) \\
    \bottomrule
\end{tabular}
\end{table}

\begin{table}[!htb]
    \centering
    \caption{IJCNN Hess evals (logistic loss): finding a loose solution, \((\epsg, \epsH) = (0.040, 0.200)\)}
    \begin{tabular}{ccccccc}
        \toprule
        \multirow{2}{*}{method} & \multicolumn{2}{c}{\(\varepsilon = 0.2\)}                                     & \multicolumn{2}{c}{\(\varepsilon = 0.6\)}  & \multicolumn{2}{c}{\(\varepsilon = 1.0\)}            \\
        \cmidrule(lr){2-3}
        \cmidrule(lr){4-5}
        \cmidrule(lr){6-7}
        & Hess evals & runtime & Hess evals & runtime & Hess evals & runtime \\ \midrule
TR      &   \(375.0 \pm 0.0\) & \( \times \) &    \(375.0 \pm 0.0\) & \( \times \) &     \(375.0 \pm 0.0\) & \( \times \) \\
TR-B    &   \(375.0 \pm 0.0\) & \( \times \) &    \(375.0 \pm 0.0\) & \( \times \) &     \(375.0 \pm 0.0\) & \( \times \) \\
\midrule
OPT     &     \(0.0 \pm 0.0\) & \( \times \) &    \(0.5 \pm 0.577\) & \( \times \) &  \(85.25 \pm 47.647\) & \( \times \) \\
OPT-B   &  \(1.0 \pm 0.0\) & \(8.6 \pm 0.2\) &  \(1.0 \pm 0.0\) & \(11.0 \pm 4.9\) &    \(1.0 \pm 0.0\) & \(9.0 \pm 1.0\) \\
OPT-LS  &     \(0.0 \pm 0.0\) & \( \times \) &      \(0.0 \pm 0.0\) & \( \times \) &      \(0.25 \pm 0.5\) & \( \times \) \\
\midrule
2OPT    &   \(6.5 \pm 1.732\) & \( \times \) &   \(1.0 \pm 0.0\) & \(0.0 \pm 0.0\) &    \(1.0 \pm 0.0\) & \(0.0 \pm 0.0\) \\
2OPT-B  &  \(1.0 \pm 0.0\) & \(1.0 \pm 0.0\) &   \(1.0 \pm 0.0\) & \(1.4 \pm 0.9\) &    \(1.0 \pm 0.0\) & \(1.1 \pm 0.3\) \\
2OPT-LS &    \(0.25 \pm 0.5\) & \( \times \) &   \(1.0 \pm 0.0\) & \(0.0 \pm 0.0\) &    \(1.0 \pm 0.0\) & \(0.0 \pm 0.0\) \\
    \bottomrule
    \end{tabular}
\end{table}

\begin{table}[!htb]
    \centering
    \caption{IJCNN (logistic loss): finding a tight solution: \((\epsg, \epsH) = (0.020, 0.141)\)}
    \begin{tabular}{ccccccc}
        \toprule
        \multirow{2}{*}{method} & \multicolumn{2}{c}{\(\varepsilon = 0.2\)}                                     & \multicolumn{2}{c}{\(\varepsilon = 0.6\)}  & \multicolumn{2}{c}{\(\varepsilon = 1.0\)}            \\
        \cmidrule(lr){2-3}
        \cmidrule(lr){4-5}
        \cmidrule(lr){6-7}
        & final loss & runtime & loss & runtime & loss & runtime \\ \midrule
   \midrule
TR      &      \(0.478 \pm 0.01\) & \( \times \) &     \(0.454 \pm 0.005\) & \( \times \) &     \(0.446 \pm 0.003\) & \( \times \) \\
TR-B    &      \(0.479 \pm 0.01\) & \( \times \) &     \(0.454 \pm 0.005\) & \( \times \) &     \(0.446 \pm 0.003\) & \( \times \) \\
\midrule
OPT     &     \(0.484 \pm 0.012\) & \( \times \) &     \(0.447 \pm 0.007\) & \( \times \) &      \(0.44 \pm 0.003\) & \( \times \) \\
OPT-B   &  \(0.501 \pm 0.007\) & \(7.9 \pm 6.3\) &  \(0.501 \pm 0.007\) & \(7.8 \pm 6.2\) &  \(0.501 \pm 0.007\) & \(9.2 \pm 7.4\) \\
OPT-LS  &     \(0.761 \pm 0.066\) & \( \times \) &     \(0.553 \pm 0.037\) & \( \times \) &     \(0.492 \pm 0.016\) & \( \times \) \\
\midrule
2OPT    &     \(0.527 \pm 0.024\) & \( \times \) &     \(0.462 \pm 0.013\) & \( \times \) &  \(0.501 \pm 0.007\) & \(0.0 \pm 0.0\) \\
2OPT-B  &  \(0.502 \pm 0.007\) & \(1.5 \pm 0.4\) &  \(0.502 \pm 0.007\) & \(1.3 \pm 0.1\) &  \(0.502 \pm 0.007\) & \(1.6 \pm 0.3\) \\
2OPT-LS &      \(3.504 \pm 0.98\) & \( \times \) &     \(0.798 \pm 0.081\) & \( \times \) &   \(0.459 \pm 0.01\) & \(0.4 \pm 0.4\) \\
        \bottomrule
    \end{tabular}
\end{table}

\begin{table}[!htb]
    \centering
    \caption{IJCNN Hess evals (logistic loss): finding a tight solution: \((\epsg, \epsH) = (0.020, 0.141)\)}
    \begin{tabular}{ccccccc}
        \toprule
        \multirow{2}{*}{method} & \multicolumn{2}{c}{\(\varepsilon = 0.2\)}                                     & \multicolumn{2}{c}{\(\varepsilon = 0.6\)}  & \multicolumn{2}{c}{\(\varepsilon = 1.0\)}            \\
        \cmidrule(lr){2-3}
        \cmidrule(lr){4-5}
        \cmidrule(lr){6-7}
        & Hess evals & runtime & Hess evals & runtime & Hess evals & runtime \\ \midrule
TR      &  \(1061.0 \pm 0.0\) & \( \times \) &  \(1061.0 \pm 0.0\) & \( \times \) &   \(1061.0 \pm 0.0\) & \( \times \) \\
TR-B    &  \(1061.0 \pm 0.0\) & \( \times \) &  \(1061.0 \pm 0.0\) & \( \times \) &   \(1061.0 \pm 0.0\) & \( \times \) \\
\midrule
OPT     &     \(0.0 \pm 0.0\) & \( \times \) &     \(0.0 \pm 0.0\) & \( \times \) &      \(0.0 \pm 0.0\) & \( \times \) \\
OPT-B   &  \(1.0 \pm 0.0\) & \(7.9 \pm 6.3\) &  \(1.0 \pm 0.0\) & \(7.8 \pm 6.2\) &   \(1.0 \pm 0.0\) & \(9.2 \pm 7.4\) \\
OPT-LS  &     \(0.0 \pm 0.0\) & \( \times \) &     \(0.0 \pm 0.0\) & \( \times \) &      \(0.0 \pm 0.0\) & \( \times \) \\
\midrule
2OPT    &     \(0.0 \pm 0.0\) & \( \times \) &   \(3.0 \pm 1.581\) & \( \times \) &   \(1.0 \pm 0.0\) & \(0.0 \pm 0.0\) \\
2OPT-B  &  \(1.0 \pm 0.0\) & \(1.5 \pm 0.4\) &  \(1.0 \pm 0.0\) & \(1.3 \pm 0.1\) &   \(1.0 \pm 0.0\) & \(1.6 \pm 0.3\) \\
2OPT-LS &     \(0.0 \pm 0.0\) & \( \times \) &   \(0.4 \pm 0.548\) & \( \times \) &  \(4.8 \pm 4.97\) & \(0.4 \pm 0.4\) \\
    \bottomrule
    \end{tabular}
\end{table}

\end{document}